%% file: neurips_2026.tex
\documentclass{article}

 \usepackage[main, final]{neurips_2026}

\usepackage[utf8]{inputenc} 
\usepackage[T1]{fontenc}    
\usepackage{hyperref}       
\usepackage{url}            
\usepackage{booktabs}       
\usepackage{amsfonts}       
\usepackage{nicefrac}       
\usepackage{microtype}      
\usepackage{xcolor}         

\usepackage{amsmath}
\usepackage{amssymb}   
\usepackage{multirow} 
\usepackage{adjustbox}
\usepackage{framed}
\usepackage{wrapfig}
\usepackage{rotating}
\usepackage{amsthm}
\usepackage{mathtools}

\newtheorem{theorem}{Theorem}

\usepackage[table]{xcolor}

\newcommand{\jygl}{\cellcolor{gray!15}}
\title{Ghosted Layers: Unconstrained Activation Alignment for Recovering Layer-Pruned LLMs}

\author{
  Vincent-Daniel Yun\textsuperscript{1}, 
  Junhyuk Jo\textsuperscript{2}, 
  Sai Praneeth Karimireddy\textsuperscript{1},
  Sunwoo Lee\textsuperscript{2}\thanks{Corresponding author. Official github repository is \url{https://github.com/daniel-eai/ghosted_layers_official_repository}.} \\
  \textsuperscript{1}University of Southern California, United States\\ 
  \textsuperscript{2}Inha University, Republic of Korea \\
  \texttt{\{yunjuyou,karimire\}@usc.edu} \\
  \texttt{\{911whwnsgur,sunwool\}@inha.ac.kr}
}

\begin{document}

\maketitle

\begin{abstract}
\input{contents/0-abstract}
\end{abstract}

\section{Introduction}
\input{contents/1-intro}

\section{Related works}
\input{contents/2-related}


\section{Method: Ghosted Layers}
\input{contents/3-method}

\section{Unconstrained solution space analysis}
\input{contents/4-analysis}

\section{Experimental results}
\input{contents/4-results}

\section{Discussion}
\input{contents/5-discussion}

\textbf{Limitations.}
\input{contents/6-limitation}

\section{Conclusion}
\input{contents/6-conclusion}

\section*{Acknowledgment}
This work was supported by Institute of Information \& communications Technology Planning \& Evaluation (IITP) under the Artificial Intelligence Innovation Human Resources Development (IITP-2026-RS-2026-25548560) grant funded by the Korea government(MSIT). This work was supported by Institute of Information \& communications Technology Planning \& Evaluation(IITP) under the Leading Generative AI Human Resources Development(IITP-2026-RS-2026-25544527) grant funded by the Korea government(MSIT). This work was supported by INHA UNIVERSITY Research Grant.

\bibliographystyle{plain}
\bibliography{main}


\appendix

\clearpage
\appendix

\input{contents/appendix}




\newpage
\input{checklist.tex}

\end{document}

%% file: contents/0-abstract.tex
Layer pruning removes entire Transformer decoder blocks from large language models, but introduces a mismatch between the hidden state received by the next surviving layer and the distribution it was trained to process, leading to significant performance degradation. We propose Ghosted Layers, a training-free recovery module that addresses this issue by solving a boundary activation alignment problem. Our method derives a closed-form optimal linear operator from a small calibration set to reconstruct the activation discrepancy introduced by the pruned layers. We show that this solution corresponds to the unconstrained optimum of the alignment objective, whereas LinearPatch is restricted to a fixed basis symmetric operator family. Across multiple LLM backbones, pruning strategies, and pruning depths, Ghosted Layers outperforms prior training-free recovery methods in accuracy and perplexity in most settings while preserving the efficiency benefits of layer pruning.

%% file: contents/1-intro.tex
Large language models (LLMs) have shown strong capabilities across a wide range of natural language tasks~\citep{brown2020gpt3, touvron2023llama, achiam2023gpt4}, but their size makes deployment costly. Various compression techniques have been explored to address this, including unstructured pruning~\citep{frantar2023sparsegpt, sun2024wanda} and layer pruning~\citep{men2024shortgpt, gromov2024unreasonable, kim2024shortened, song2024sleb}. Among these, layer pruning stands out for its practicality, as it removes entire Transformer decoder blocks and yields a smaller model that runs on standard inference stacks without custom kernels or architectural changes. 

\begin{wrapfigure}{r}{0.5\textwidth}
\vspace{-0.6cm}
    \centering
    \includegraphics[width=\linewidth]{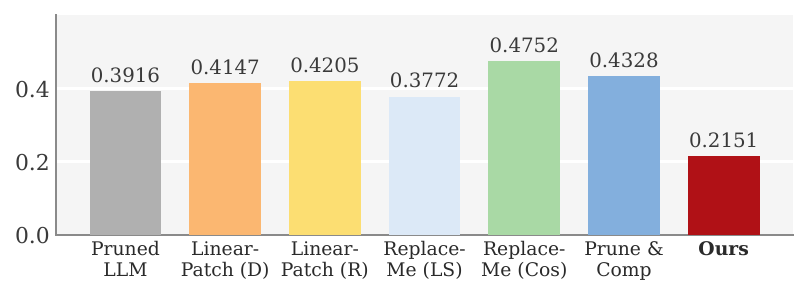}
    \vspace{-0.2cm}
    \caption{Boundary activation MAE on LLaMA-3.1-8B with 7 layers pruned by LLM-Streamline~\citep{chen2024streamline}, comparing existing recovery methods with Ghosted Layers.}
    \label{fig:boundary_mae}
\vspace{-0.6cm}

\end{wrapfigure}

Yet pruned models often exhibit substantial accuracy degradation: pruning a block of layers eliminates the intermediate transformations between the surviving layers, causing a distribution shift at the pruning boundary that propagates through downstream layers. This motivates the need for a recovery mechanism to compensate for the missing layers.

Recent work mitigates this through a variety of training-free recovery modules. \texttt{ReplaceMe}~\citep{shopkhoev2025replaceme} approximates the pruned block's computation with a linear transformation absorbed into the surviving weights, but does not directly address the mismatch at the pruning boundary. In contrast, \texttt{LinearPatch}~\citep{lp} directly targets the \emph{boundary activation mismatch}, the discrepancy between the hidden state expected by the next surviving layer and the one it actually receives, by inserting a linear operator at the pruning boundary. This operator is implemented as a Hadamard-rotated channel-wise scaling and is symmetric by construction, restricting it to a strict subspace of linear operators and preventing it from reaching the optimum of the alignment objective in the full operator space.

To empirically verify this limitation, we quantify the boundary activation mismatch via the \emph{boundary activation error}, defined as the mean absolute error (MAE) between the post-boundary hidden state of the original model and the input received by the next layer in the pruned model. Figure~\ref{fig:boundary_mae} compares representative post-pruning recovery methods on the resulting boundary activation error. Existing methods leave a substantial portion of the error uncorrected, suggesting that the boundary activation mismatch is not entirely resolved after recovery.


We address this limitation by deriving the closed-form unconstrained optimum of the boundary activation alignment objective over the full space of linear operators. In contrast, LinearPatch corresponds to a constrained solution restricted to a specific operator subspace. We further show that the optimal operator contains a substantial anti-symmetric component, which is structurally inaccessible to symmetric constructions such as \texttt{LinearPatch}.

Building on this analysis, we propose \texttt{Ghosted Layers}, a training-free recovery module that inserts the closed-form optimal operator at the pruning boundary via a forward hook. It is compatible with any pruning criterion and model architecture, and consistently improves recovery across multiple LLM backbones and benchmarks. These results suggest that solving boundary activation alignment at its unconstrained optimum is sufficient to substantially recover layer-pruned LLMs without retraining.

\textbf{Contribution} of this study can be summarized as follows:
\begin{itemize}
    \item We formulate post-pruning recovery as an activation alignment problem and derive its closed-form optimal solution from boundary activations.
    \item We show that this formulation yields an unconstrained optimum that includes a substantial anti-symmetric component, which is inaccessible to symmetric constructions such as \texttt{LinearPatch}.
    \item We propose \texttt{Ghosted Layers}, a training-free and plug-and-play recovery module compatible with any pruning criterion and model architecture, which consistently outperforms prior methods at matched inference cost.
\end{itemize}

%% file: contents/2-related.tex
\textbf{Structured pruning of LLMs.}
Structured pruning reduces model size by removing groups of parameters while preserving dense computation. Width pruning removes attention heads, MLP neurons, or hidden dimensions using importance scores derived from weights~\citep{ma2023llmpruner}, activations~\citep{an2024flap, ashkboos2024slicegpt}, or gradients~\citep{xia2024sheared}, but often requires retraining or distillation to recover performance~\citep{muralidharan2024compact}. Layer pruning instead removes entire Transformer blocks while preserving the original dense architecture and requiring no specialized hardware support. ShortGPT~\citep{men2024shortgpt} performs one-shot pruning using cosine similarity between block inputs and outputs, SLEB~\citep{song2024sleb} iteratively removes layers with minimal impact on perplexity, and Shortened LLaMA~\citep{kim2024shortened} ranks layers by the perplexity increase induced by their removal. LLM-Streamline~\citep{chen2024streamline} selects a \emph{contiguous} block of layers based on boundary activation similarity. These methods primarily address \emph{which} layers to remove, while the activation mismatch introduced at the resulting pruning boundaries remains a separate recovery problem.

\textbf{Post-pruning recovery.}
Post-pruning recovery methods aim to compensate for the computation removed by layer pruning. Prune\&Comp~\citep{chen2025prunecomp} rescales surviving boundary weights using per-channel factors, capturing channel-wise magnitude changes without modeling cross-channel interactions. \texttt{ReplaceMe}~\citep{shopkhoev2025replaceme} permits a full linear transformation in its primary formulation, but reconstructs the missing computation from the preceding MLP branch rather than directly from the pruning-boundary hidden state. In contrast, \texttt{LinearPatch}~\citep{lp} operates directly on the pruning-boundary state, but parameterizes recovery through Hadamard-rotated diagonal scaling, restricting the operator to a fixed-basis symmetric family. Therefore, these methods impose complementary restrictions on boundary recovery, either through the representation used for reconstruction or through the class of transformations available to the recovery operator. \texttt{Ghosted Layers} addresses both limitations by reconstructing the missing update directly from the pruning-boundary state using an unconstrained linear operator whose solution is obtained in closed form.

%% file: contents/3-method.tex
\label{sec:method}

Layer pruning removes transformer blocks to reduce model size, but the hidden state passed to the next surviving layer no longer matches the one it was trained to process, causing an activation mismatch at the pruning boundary that degrades downstream performance. We propose \textit{Ghosted Layers}, a training-free recovery method that mitigates this mismatch via a closed-form linear operator inserted at each pruning boundary (Figure~\ref{fig:main}).

\begin{figure*}[h]
  \centering
  \includegraphics[width=\linewidth]{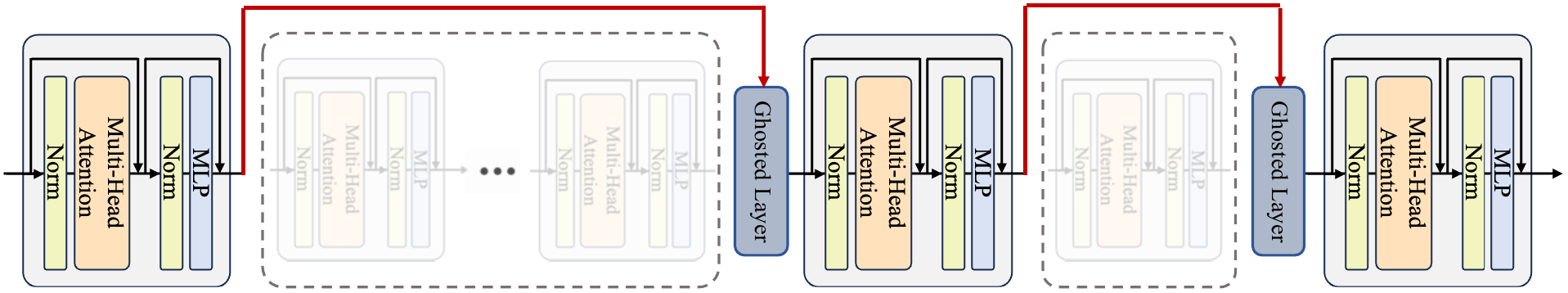}
  \caption{
    Ghosted Layers as drop-in replacements for pruned transformer blocks. One or more consecutive transformer blocks are removed (dashed), and a single Ghosted Layer (blue) takes their place at the boundary. The red arrows show the hidden state bypassing the pruned region and flowing through the Ghosted Layer before reaching the next surviving block. This applies to both contiguous and non-contiguous pruning, with exactly one Ghosted Layer substituting each pruned region.
  }
  \label{fig:main}
\end{figure*}

\subsection{Notation and setup}
\label{sec:notation}

Let $\mathcal{M} = \{f^{(\ell)}\}_{\ell=0}^{L-1}$ be a pre-trained autoregressive LLM with $L$ Transformer decoder layers, where $f^{(\ell)} : \mathbb{R}^{B \times T \times C} \to \mathbb{R}^{B \times T \times C}$ and $B$, $T$, $C$ denote batch size, sequence length, and hidden dimension. Under the standard pre-norm residual architecture:
\begin{equation}
    \mathbf{X}^{(\ell+1)}
    =
    \mathbf{X}^{(\ell)}
    +
    f^{(\ell)}\!\bigl(\mathbf{X}^{(\ell)};\,\theta^{(\ell)}\bigr),
    \qquad \ell = 0, \ldots, L-1.
    \label{eq:residual}
\end{equation}
Layer pruning removes a contiguous block $\mathcal{B} = \{\ell^*, \ldots, \ell^*\!+\!n\!-\!1\}$ of $n$ layers, where $\ell^*$ and $\ell^*\!+\!n$ are the \emph{pre-boundary} and \emph{post-boundary} indices. Unrolling Eq.~\eqref{eq:residual} over $\mathcal{B}$, the post-boundary hidden state in the original model satisfies:
\begin{equation}
    \mathbf{X}^{(\ell^*+n)}
    =
    \mathbf{X}^{(\ell^*)}
    +
    \underbrace{
        \sum_{k=\ell^*}^{\ell^*+n-1}
        f^{(k)}\!\bigl(\mathbf{X}^{(k)};\,\theta^{(k)}\bigr)
    }_{\displaystyle \boldsymbol{\Delta}^{(\ell^*,\,n)}}.
    \label{eq:unroll}
\end{equation}
The pruned model sets $\mathbf{X}^{(\ell^*+n)}_{\mathrm{pruned}} = \mathbf{X}^{(\ell^*)}$, so the activation received by the first surviving layer differs from $\mathbf{X}^{(\ell^*+n)}$ by exactly $\boldsymbol{\Delta}^{(\ell^*,n)}$. All downstream layers, which were trained to process $\mathbf{X}_{\mathrm{post}}$, instead receive $\mathbf{X}_{\mathrm{pre}}$, and this activation mismatch propagates through the network.

Given a calibration corpus $\mathcal{D}$ of $N$ sequences of length $T$ (so $T_{\mathcal{D}} = NT$ tokens in total), we collect the hidden states at the two boundary layers from the original model: $\mathbf{X}_{\mathrm{pre}}, \mathbf{X}_{\mathrm{post}} \in \mathbb{R}^{T_{\mathcal{D}} \times C}$, where rows correspond to flattened tokens. We define the \emph{boundary activation gap}:
\begin{equation}
    \boldsymbol{\Delta}
    \triangleq
    \mathbf{X}_{\mathrm{post}} - \mathbf{X}_{\mathrm{pre}}
    \in \mathbb{R}^{T_{\mathcal{D}} \times C},
    \label{eq:delta}
\end{equation}
which is the empirical estimate of $\boldsymbol{\Delta}^{(\ell^*,n)}$ over $\mathcal{D}$, and quantifies the activation mismatch that any recovery method must reduce.

\subsection{Ghosted Layers}

\textit{Ghosted Layers} is a calibration-based plug-and-play module that reduces the boundary activation mismatch via a closed-form optimal linear operator. Given a small calibration set $\mathcal{D}$, the method proceeds in three offline steps: collect the boundary activations, compute the optimal linear operator in closed form, and insert it into the pruned model.

\subsubsection{Collecting boundary activations.}
\label{sec:step1}
We run the original model $\mathcal{M}$ on $\mathcal{D}$ and capture the hidden states at the two boundary layers via forward pre-hooks. A forward pre-hook fires immediately before a layer's computation, so $\mathbf{X}_{\mathrm{pre}} \in \mathbb{R}^{T_{\mathcal{D}} \times C}$ captures the input to the first pruned layer $\ell^*$, and $\mathbf{X}_{\mathrm{post}} \in \mathbb{R}^{T_{\mathcal{D}} \times C}$ captures the input to the first surviving layer $\ell^*\!+\!n$. The boundary activation gap is then:
$
    \boldsymbol{\Delta}
    =
    \mathbf{X}_{\mathrm{post}} - \mathbf{X}_{\mathrm{pre}}
    \in \mathbb{R}^{T_{\mathcal{D}} \times C},
    \label{eq:delta_step1}
$
which equals $\boldsymbol{\Delta}^{(\ell^*,n)}$ from Eq.~\eqref{eq:unroll} evaluated over $\mathcal{D}$.

\subsubsection{Closed-form optimal operator.}
\label{sec:step2}
We seek a linear operator $\mathbf{W} \in \mathbb{R}^{C \times C}$ that, when applied to the pre-boundary state, best approximates the post-boundary state of the unpruned model:
\begin{equation}
    \mathbf{W}^{*}
    =
    \arg\min_{\mathbf{W} \in \mathbb{R}^{C \times C}}
    \bigl\|\mathbf{X}_{\mathrm{pre}}\mathbf{W} - \mathbf{X}_{\mathrm{post}}\bigr\|_F^2.
    \label{eq:alignment_obj}
\end{equation}
Reparameterizing $\mathbf{W} = \mathbf{I} + \mathbf{M}$ and substituting into Eq.~\eqref{eq:alignment_obj} yields an equivalent objective:
\begin{equation}
    \min_{\mathbf{M} \in \mathbb{R}^{C \times C}}
    \bigl\|\mathbf{X}_{\mathrm{pre}}\mathbf{M} - \boldsymbol{\Delta}\bigr\|_F^2.
    \label{eq:M_obj}
\end{equation}
This reparameterization transforms the alignment problem into directly regressing the boundary activation gap $\boldsymbol{\Delta} = \mathbf{X}_{\mathrm{post}} - \mathbf{X}_{\mathrm{pre}}$ from the pre-boundary state $\mathbf{X}_{\mathrm{pre}}$. Rather than learning to reproduce the full post-boundary activation $\mathbf{X}_{\mathrm{post}}$, the operator $\mathbf{M}$ only needs to capture the incremental change induced by the pruned block. This decoupling isolates the target of learning to the quantity that actually differs between the pruned and unpruned models.

The minimum-norm least-squares solution to Eq.~\eqref{eq:M_obj} is:
\begin{equation}
    \mathbf{M}^{*}
    =
    \mathbf{X}_{\mathrm{pre}}^{\dagger}\,\boldsymbol{\Delta}
    =
    \mathbf{V}\,\boldsymbol{\Sigma}^{\dagger}\,\mathbf{U}^{\top}\,\boldsymbol{\Delta}
    \in \mathbb{R}^{C \times C},
    \label{eq:mstar_method}
\end{equation}
where $\mathbf{X}_{\mathrm{pre}} = \mathbf{U}\boldsymbol{\Sigma}\mathbf{V}^{\top}$ is the thin SVD with $\sigma_1 \geq \cdots \geq \sigma_C \geq 0$, and:
\begin{equation}
    \bigl(\boldsymbol{\Sigma}^{\dagger}\bigr)_{ii}
    =
    \begin{cases}
        \sigma_i^{-1} & \text{if } \sigma_i > \epsilon\,\sigma_1, \\
        0             & \text{otherwise,}
    \end{cases}
    \qquad \epsilon = 10^{-6}.
    \label{eq:sigma_dag_method}
\end{equation}
The threshold $\epsilon$ controls numerical stability and has no effect on the solution for well-conditioned $\mathbf{X}_{\mathrm{pre}}$. $\mathbf{M}^*$ is a full $C \times C$ matrix with no structural constraint.

Then, the Ghosted Layers operator is
$
    \mathbf{W}^{*}
    =
    \mathbf{I} + \mathbf{M}^{*}
    \in \mathbb{R}^{C \times C}.
$

\subsubsection{Inserting the Ghosted Layers operator.}
\label{sec:step3}
After removing $\mathcal{B}$ and re-indexing the surviving layers, we insert $\mathbf{W}^*$ at the output of layer $\ell^*\!-\!1$ via a forward hook. Given a hidden state $\mathbf{x} \in \mathbb{R}^{B \times T \times C}$, the module computes:
\begin{equation}
    \mathbf{x}_{\mathrm{new}}
    =
    \mathbf{x}\,\mathbf{W}^{*}
    =
    \underbrace{\mathbf{x}}_{\text{identity}}
    +
    \underbrace{\mathbf{x}\,\mathbf{M}^{*}}_{\text{additive}}\,,
    \label{eq:ghosted_forward}
\end{equation}
where the additive term $\mathbf{x}\,\mathbf{M}^*$ serves as the learned estimate of the boundary gap $\boldsymbol{\Delta}$: since $\mathbf{M}^*$ minimizes $\|\mathbf{X}_{\mathrm{pre}}\mathbf{M} - \boldsymbol{\Delta}\|_F^2$, the product $\mathbf{x}\,\mathbf{M}^*$ approximates the incremental update that the pruned layers would have contributed to $\mathbf{x}$. When $\mathbf{M}^* = \mathbf{0}$, Eq.~\eqref{eq:ghosted_forward} reduces to $\mathbf{x}_{\mathrm{new}} = \mathbf{x}$, recovering the pruned baseline exactly.

\textbf{Practical.}
We approximate $\mathbf{M}^{*}$ by solving a regularized least-squares system via \texttt{torch.linalg.solve}, using $\epsilon = 10^{-6}$. The resulting correction $\widehat{\mathbf{M}}_{\epsilon}$ converges to $\mathbf{X}_{\mathrm{pre}}^{\dagger}\boldsymbol{\Delta}$ as $\epsilon\to0$ and yields nearly identical downstream results to the SVD-based implementation in our comparison. We use $\mathbf{I}+\widehat{\mathbf{M}}_{\epsilon}$ for inference; implementation details are provided in Appendix~\ref{app:solver_comparison}.

%% file: contents/4-analysis.tex
\label{sec:analysis}
We empirically analyze the properties of the unconstrained optimal operator $\mathbf{M}^{*}$ to validate the theoretical claims established in Section~\ref{sec:method}. All experiments in this section use $n{=}7$ pruned layers across two LLM backbones; detailed experimental settings are provided in Appendix~\ref{app:settings}.

Theorem~\ref{thm:main_main} establishes that $\mathbf{W}^{*}$ is the unconstrained minimizer of LinearPatch's own alignment objective, and that LinearPatch's symmetric parameterization prevents it from attaining this solution.

\begin{theorem}[Ghosted Layers is the Unconstrained Optimum of LinearPatch]
\label{thm:main_main}
Let $\mathbf{X}_{\mathrm{pre}}, \mathbf{X}_{\mathrm{post}} \in \mathbb{R}^{T_{\mathcal{D}} \times C}$ and $\boldsymbol{\Delta} = \mathbf{X}_{\mathrm{post}} - \mathbf{X}_{\mathrm{pre}}$. Both LinearPatch and Ghosted Layers produce a repaired activation of the form $\mathbf{X}_{\mathrm{new}} = \mathbf{X}_{\mathrm{pre}} \mathbf{W}$, but differ in the structure of $\mathbf{W}$:
\begin{align}
    \mathbf{X}_{\mathrm{new, LP}} = \mathbf{X}_{\mathrm{pre}} \cdot 
    \underbrace{\mathbf{H}\mathbf{D}\mathbf{H}^{\top}}_{\mathbf{W}_{\mathrm{LP}},\ \mathbf{W}_{\mathrm{LP}}^{\top} = \mathbf{W}_{\mathrm{LP}}} \quad 
    & \mathbf{X}_{\mathrm{new, GL}} = \mathbf{X}_{\mathrm{pre}} \cdot 
    \underbrace{(\mathbf{I} + \mathbf{M}^{*})}_{\mathbf{W}^{*},\ \mathbf{W}^{*} \in \mathbb{R}^{C \times C}}
\end{align}
where $\mathbf{M}^{*} = \mathbf{X}_{\mathrm{pre}}^{\dagger}\boldsymbol{\Delta}$, and $\mathbf{W}^{*} = \mathbf{I} + \mathbf{M}^{*}$ is the minimum-norm solution to the unconstrained alignment objective $\min_{\mathbf{W}} \|\mathbf{X}_{\mathrm{pre}}\mathbf{W} - \mathbf{X}_{\mathrm{post}}\|_F^2$, whereas $\mathbf{W}_{\mathrm{LP}}$ searches only over symmetric matrices, making LinearPatch a constrained approximation to Ghosted Layers.
\end{theorem}

\begin{wrapfigure}{r}{0.5\textwidth}
    \centering
    \vspace{-0.5cm}
    \includegraphics[width=\linewidth]{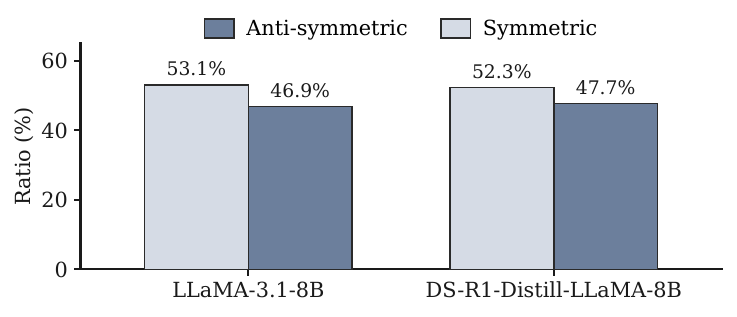}
    \caption{Squared Frobenius norm decomposition of
    $\mathbf{M}^{*}$ into symmetric and anti-symmetric
    components across two LLM backbones
    ($n=7$, LLM-Streamline). Detailed setups are in Appendix~\ref{app:sym_decomp}}
    \label{fig:symm}
\end{wrapfigure}

Theorem~\ref{thm:main_main} thus positions LinearPatch as a constrained special case of the activation alignment framework: both methods optimize the same alignment objective over the same functional form $\mathbf{X}_{\mathrm{new}} = \mathbf{X}_{\mathrm{pre}}\mathbf{W}$, yet LinearPatch confines its search to the symmetric subspace of dimension $\frac{C(C+1)}{2}$, rendering $\mathbf{W}^{*}$ structurally inaccessible whenever it has a non-zero anti-symmetric component. The proof is deferred to Appendix~\ref{app:proof}.

\textbf{Constrained solution space.}
As established in Theorem~\ref{thm:main_main}, $\mathbf{W}^{*}$ is the unconstrained minimizer over all of $\mathbb{R}^{C \times C}$, whereas any symmetric operator $\mathbf{W}$ satisfies $\mathbf{W} - \mathbf{W}^{\top} = \mathbf{0}$ and is thus confined to the symmetric subspace. To empirically verify that $\mathbf{W}^{*}$ indeed lies outside this subspace, we compute $\mathbf{M}^{*}$ from the calibration activations $\mathbf{X}_{\mathrm{pre}}, \mathbf{X}_{\mathrm{post}}$ collected from the original unpruned model, and decompose $\mathbf{M}^{*}$ into its symmetric and anti-symmetric components:
\begin{equation}
    \mathbf{M}^{*} = 
    \underbrace{\frac{\mathbf{M}^{*} + (\mathbf{M}^{*})^{\top}}{2}}_{\mathbf{M}^{*}_{\mathrm{sym}}}
    +
    \underbrace{\frac{\mathbf{M}^{*} - (\mathbf{M}^{*})^{\top}}{2}}_{\mathbf{M}^{*}_{\mathrm{asym}}}.
    \label{eq:decomp}\notag
\end{equation}
Figure~\ref{fig:symm} shows that $\mathbf{M}^{*}_{\mathrm{asym}}$ is non-zero and comparable in magnitude to $\mathbf{M}^{*}_{\mathrm{sym}}$ consistently across all two backbones. Since any symmetric operator satisfies $\mathbf{M}_{\mathrm{asym}} = \mathbf{0}$ by construction, this anti-symmetric component is structurally inaccessible to LinearPatch regardless of how its diagonal $\mathbf{D}$ is chosen.

\begin{figure}[t]
\centering
\includegraphics[width=\linewidth]{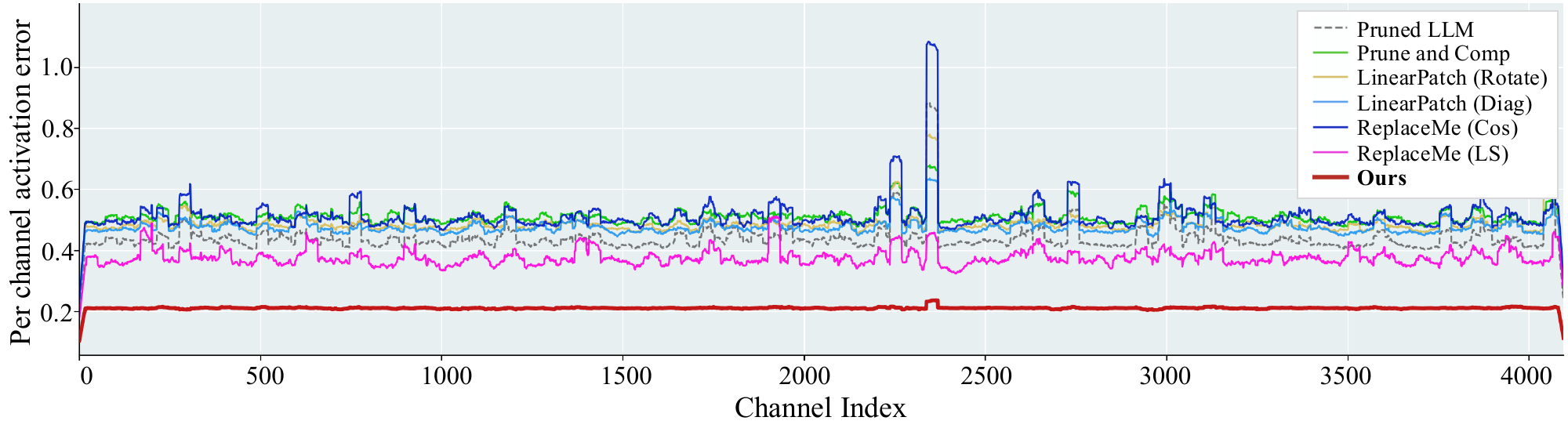}
\caption{
Per-channel mean absolute error (MAE) between the repaired boundary activation $\mathbf{X}_{\mathrm{pre}}\mathbf{W}$ and the target post-boundary activation $\mathbf{X}_{\mathrm{post}}$ of the unpruned model, computed on LLaMA-3.1-8B with $n{=}7$ layers removed via LLM-Streamline. We use 128 sequences of length 2{,}048 sampled from the C4 training split for layer selection and baseline recovery, and 32 sequences for \texttt{Ghosted Layers} operator estimation. Each curve shows the MAE averaged over tokens for each of the $C{=}4{,}096$ channels. The Pruned LLM curve corresponds to the LLM-Streamline-pruned model without any recovery operator applied.
}
\label{fig:channel_mae}
\end{figure}

\setlength{\tabcolsep}{4.5pt}
\begin{table*}[t]
\centering
\scriptsize
\caption{Zero-shot accuracy (\%) on commonsense QA benchmarks across three LLM backbones, with 7/32 and 11/32 pruning for the 8B models and 13/40 and 15/40 for Qwen3-14B. All methods are training-free. Recovery methods use LLM-Streamline; for ReplaceMe and LinearPatch, we report the variant with the highest AVG in each setting. AVG denotes the mean accuracy across all nine tasks. $L_p/L_t$ denotes the number of pruned layers over the total number of layers.}
\begin{adjustbox}{width=1\columnwidth}
\begin{tabular}{l  l  l c c c c c c c c c  r }
\toprule
\textbf{Model} & \textbf{$L_p/L_t$ } & \textbf{Method }&  \textbf{ARC-E}& \textbf{ARC-C }& \textbf{HellaS} & \textbf{WinoG} & \textbf{BoolQ}  & \textbf{OBQA} & \textbf{RTE} & \textbf{COPA} & \textbf{Race} & \textbf{AVG} \textcolor{red}{$\uparrow$} \\
\midrule
\multirow{16}{*}{\begin{sideways}LLaMA-3.1-8B\end{sideways}}
& 0/32 & Dense          &  81.19 & 53.41 & 78.92 & 73.64 & 82.11 & 44.80 & 69.68 & 87.00 & 39.14 & 67.77\\
\cmidrule{2-13}
& 7/32 & Shortened LLaMA & 61.32 & 33.11 & 59.55 & 54.22 & 43.76 & 35.20 & 51.62 & 77.00 & 31.29 & 49.67\\
& 7/32 & LLM-Streamline & 44.19 & 33.11 & 33.39 & 56.99 & 38.20 & 32.60 & 58.12 & 61.00 & 25.84 & 42.60 \\
& 7/32 & ShortGPT & 58.29 & 42.15 & 64.96 & 68.35 & 61.99 & 34.40 & 69.68 & 80.00 & 34.45 & 57.14 \\
& 7/32 & Prune\&Comp & 46.25 & 30.89 & 44.22 & 59.12 & 53.91 & 35.00 & 60.29 & 67.00 & 28.61 & 47.25 \\
& 7/32 & ReplaceMe (LS) & 64.90 & 43.52 & 63.80 & 71.67 & 69.60 & 37.80 & 71.48 & 77.00 & 37.32 & 59.68\\
& 7/32 & LinearPatch (Rotate) & 57.62 & 37.29 & 54.91 & 65.43 & 60.49 & 35.80 & 69.68 & 69.00 & 29.19 & 53.27 \\
& 7/32 & \jygl Ghosted Layers (Ours) & \jygl 68.01 & \jygl 43.00 & \jygl 66.60 & \jygl 71.59 & \jygl 71.31 & \jygl 37.20 & \jygl 68.59 & \jygl 76.00 & \jygl 37.80 & \jygl 60.01  \\
\cmidrule{2-13}
& 11/32 & Shortened LLaMA  & 38.54 & 30.38 & 29.05 & 50.91 & 59.80 & 28.20 & 50.26 & 62.00 & 29.38 & 42.05 \\
& 11/32 & LLM-Streamline & 39.69 & 30.29 & 31.49 & 56.12 & 55.11 & 30.40 & 69.68 & 61.00 & 28.33 & 44.68 \\
& 11/32 & ShortGPT & 39.69 & 30.29 & 31.49 & 56.12 & 55.11 & 30.40 & 69.68 & 61.00 & 28.33 & 44.68 \\
& 11/32 & Prune\&Comp & 38.51 & 27.30 & 40.50 & 53.28 & 62.08 & 29.40 & 53.79 & 57.00 & 25.17 & 43.00 \\
& 11/32 & ReplaceMe (LS) & 44.53 & 34.04 & 47.23 & 67.40 & 73.55 & 29.80 & 70.76 & 67.00 & 29.95 & 51.58\\
& 11/32 & LinearPatch (Rotate) & 50.93 & 34.98 & 47.30 & 62.19 & 74.65 & 30.40 & 71.48 & 68.00 & 31.58 & 52.39\\
& 11/32 & \jygl Ghosted Layers (Ours) & \jygl 50.08 & \jygl 34.64 & \jygl 52.73 & \jygl 69.30 & \jygl 72.63 & \jygl 31.80 & \jygl 69.68 & \jygl 72.00 & \jygl 32.44 & \jygl 53.92\\
\midrule
\multirow{16}{*}{\begin{sideways}DeepSeek-R1-Distill-LLaMA-8B\end{sideways}}
& 0/32 & Dense & 65.87 & 42.41 & 74.35 & 67.80 & 82.91 & 41.20 & 69.68 & 89.00 & 41.63 & 63.87\\
\cmidrule{2-13}
& 7/32 & Shortened LLaMA & 45.71 & 29.69 & 55.66 & 57.70 & 64.65 & 30.40 & 50.54 & 76.00 & 30.62 & 49.00\\
& 7/32 & LLM-Streamline & 49.92 & 35.24 & 46.33 & 61.56 & 51.99 & 32.40 & 57.40 & 70.00 & 27.27 & 48.01\\
& 7/32 & ShortGPT & 49.12 & 37.03 & 55.95 & 63.06 & 77.09 & 34.00 & 74.73 & 71.00 & 33.21 & 55.02\\
& 7/32 & Prune\&Comp & 44.99 & 31.06 & 43.44 & 56.12 & 53.12 & 31.80 & 60.65 & 64.00 & 26.79 & 45.77 \\
& 7/32 & ReplaceMe (LS) & 55.35 & 37.03 & 59.97 & 66.46 & 69.08 & 35.20 & 75.09 & 80.00 & 37.80 & 57.33\\
& 7/32 & LinearPatch (Rotate) & 54.08 & 36.18 & 53.29 & 64.25 & 72.97 & 33.80 & 62.82 & 72.00 & 32.92 & 53.59 \\
& 7/32 & \jygl Ghosted Layers (Ours) & \jygl 55.81 & \jygl 37.03 & \jygl 61.30 & \jygl 67.25 & \jygl 68.47 & \jygl 37.00 & \jygl 72.92 & \jygl 81.00 & \jygl 39.43 & \jygl 57.80\\
\cmidrule{2-13}
& 11/32 & Shortened LLaMA  & 41.75 & 28.84 & 42.84 & 54.30 & 55.23 & 27.80 & 53.79 & 70.00 & 26.99 & 44.62 \\
& 11/32 & LLM-Streamline & 37.08 & 31.66 & 38.72 & 55.17 & 75.20 & 27.40 & 64.26 & 63.00 & 26.22 & 46.52 \\
& 11/32 & ShortGPT & 37.08 & 31.66 & 38.72 & 55.17 & 75.20 & 27.40 & 64.26 & 63.00 & 26.22 & 46.52 \\
& 11/32 & Prune\&Comp & 36.45 & 29.10 & 36.32 & 50.99 & 64.62 & 27.40 & 60.29 & 56.00 & 24.40 & 42.84 \\
& 11/32 & ReplaceMe (LS) & 40.19 & 30.63 & 44.70 & 61.40 & 66.45 & 32.00 & 74.73 & 65.00 & 33.78 & 49.88\\
& 11/32 & LinearPatch (Diag) & 45.58 & 34.90 & 43.97 & 58.33 & 77.46 & 31.60 & 68.23 & 64.00 & 29.57 & 50.40 \\
& 11/32 & \jygl Ghosted Layers (Ours) & \jygl 42.72 & \jygl 32.59 & \jygl 48.09 & \jygl 64.33 & \jygl 75.41 & \jygl 31.80 & \jygl 74.01 & \jygl 66.00 & \jygl 34.07 & \jygl 52.11\\
\midrule
\multirow{16}{*}{\begin{sideways}Qwen3-14B\end{sideways}}
&0/40& Dense &  82.83 & 60.24 & 78.82 & 72.85 & 89.30 & 46.20 & 77.62 & 90.00 & 43.16 & 71.22\\
\cmidrule{2-13}
&13/40& Shortened LLaMA  & 31.35 & 28.16 & 43.93 & 49.96 & 51.01 & 29.20 & 54.51 & 72.00 & 29.38 & 43.27  \\
&13/40& LLM-Streamline & 33.80 & 31.31 & 32.16 & 56.91 & 62.17 & 30.00 & 48.74 & 64.00 & 25.26 & 42.71\\
&13/40& ShortGPT & 29.59 & 26.71 & 37.82 & 49.49 & 62.02 & 28.80 & 49.82 & 61.00 & 24.88 & 41.13  \\
&13/40& Prune\&Comp & 33.80 & 31.31 & 32.16 & 56.91 & 62.17 & 30.00 & 48.74 & 64.00 & 25.26 & 42.71\\
&13/40& ReplaceMe (LS) & 32.32 & 27.47 & 30.69 & 58.17 & 78.17 & 27.80 & 56.32 & 64.00 & 26.41 & 44.59  \\
&13/40& LinearPatch (Diag) & 25.08 & 22.70 & 25.04 & 49.57 & 37.83 & 27.60 & 52.71 & 55.00 & 25.93 & 35.72 \\
&13/40& \jygl Ghosted Layers (Ours) & \jygl 35.65 & \jygl 28.92 & \jygl 33.91 & \jygl 60.46 & \jygl 72.84 & \jygl 29.00 & \jygl 48.38 & \jygl 66.00 & \jygl 28.04 & \jygl 44.80 \\
\cmidrule{2-13}
&15/40& Shortened LLaMA  &  35.27 & 23.04 & 34.16 & 50.36 & 52.87 & 27.40 & 55.23 & 59.00 & 26.70 & 40.45\\
&15/40& LLM-Streamline & 39.48 & 27.39 & 35.85 & 48.22 & 38.50 & 29.00 & 50.18 & 61.00 & 21.05 & 38.96 \\
&15/40& ShortGPT & 29.04 & 26.54 & 33.49 & 48.70 & 61.68 & 28.60 & 52.35 & 55.00 & 23.83 & 39.91\\
&15/40& Prune\&Comp & 39.48 & 27.39 & 35.85 & 48.22 & 38.50 & 29.00 & 50.18 & 61.00 & 21.05 & 38.96\\
&15/40& ReplaceMe (Cos) & 42.72 & 27.82 & 38.41 & 47.51 & 38.26 & 30.00 & 50.18 & 62.00 & 23.83 & 40.08\\
&15/40& LinearPatch (Rotate) & 37.58 & 26.96 & 35.69 & 50.67 & 43.58 & 29.60 & 49.46 & 55.00 & 25.93 & 39.39\\
&15/40& \jygl Ghosted Layers (Ours) & \jygl 44.78 & \jygl 27.39 & \jygl 43.57 & \jygl 53.91 & \jygl 43.39 & \jygl 30.60 & \jygl 51.99 & \jygl 67.00 & \jygl 29.86 & \jygl 43.61 \\
\bottomrule
\end{tabular}
\end{adjustbox}
\label{tab:commonsense_main}
\end{table*}

\textbf{Per-channel activation error.}
To examine how each method addresses the boundary activation mismatch along the channel dimension, we measure the per-channel MAE between the repaired activation $\mathbf{X}_{\mathrm{pre}}\mathbf{W}$ and the target $\mathbf{X}_{\mathrm{post}}$ of the unpruned model (Figure~\ref{fig:channel_mae}). Existing recovery methods largely track the pruned baseline across channels, and in some cases even exceed it, suggesting that their operators only indirectly affect the activation gap. Ghosted Layers, in contrast, directly targets the activation mismatch through its closed-form unconstrained operator, and consequently reduces the per-channel error substantially and uniformly across the hidden dimension.

\textbf{Inference Cost Equivalence with LinearPatch.}
Despite their different operator structures, \texttt{LinearPatch} and \texttt{Ghosted Layers} incur comparable inference costs at the pruning boundary. As noted by LinearPatch~\citep{lp}, \texttt{LinearPatch} fuses $\mathbf{H}\mathbf{D}\mathbf{H}^{\top}$ offline into a single dense $C \times C$ matrix, requiring only one matrix multiplication during inference. Both methods can be implemented as a single $C \times C$ matrix multiplication:
\begin{equation}
    \mathbf{X}_{\mathrm{new, LP}} = \mathbf{X}_{\mathrm{pre}} \, \mathbf{P},
    \qquad
    \mathbf{X}_{\mathrm{new, GL}} = \mathbf{X}_{\mathrm{pre}} \, \mathbf{W}^{*},
    \qquad
    \mathbf{P}, \mathbf{W}^{*} \in \mathbb{R}^{C \times C}.
\end{equation}
Although $\mathbf{D}$ contains only $C$ free parameters, the fused $\mathbf{P} = \mathbf{H}\mathbf{D}\mathbf{H}^{\top}$ and $\mathbf{W}^{*}$ have the same dense operator dimensions and require the same number of GEMM calls. \texttt{Ghosted Layers} thus provides a strictly richer operator space at inference cost comparable to \texttt{LinearPatch}.

%% file: contents/4-results.tex
\label{sec:experiments}

\subsection{Experimental setups}

\begin{table*}[t]
\centering
\footnotesize
\caption{
Comparison on PPL benchmark with training-free methods on LLaMA-3.1-8B. $L_p/L_t$ denotes the number of pruned layers $L_p$ over the total number of layers $L_t$ in the original model. Additional results across multiple LLM backbones are provided in Appendix~\ref{app:llms}.
}
\begin{adjustbox}{width=1\columnwidth}
\begin{tabular}{l r r r r r r r r r r r}
\toprule
& \multicolumn{5}{c}{\textbf{7-layer}} && \multicolumn{5}{c}{\textbf{11-Layer}}\\
\cmidrule{2-6} \cmidrule{8-12}
\textbf{Method}  & \textbf{WIKI}\textcolor{red}{$\downarrow$} &  \textbf{C4}\textcolor{red}{$\downarrow$} & \textbf{PTB}\textcolor{red}{$\downarrow$} & \textbf{PPL AVG}\textcolor{red}{$\downarrow$} & \textbf{ACC AVG}\textcolor{red}{$\uparrow$}   &&   \textbf{WIKI}\textcolor{red}{$\downarrow$} & \textbf{C4}\textcolor{red}{$\downarrow$} & \textbf{PTB}\textcolor{red}{$\downarrow$} & \textbf{PPL AVG}\textcolor{red}{$\downarrow$} & \textbf{ACC AVG}\textcolor{red}{$\uparrow$}   \\
\midrule
Dense & 6.24 & 8.68 & 10.58 & 8.50 & 67.77 && 6.24 & 8.68 & 10.58 & 8.50 & 67.77 \\
\cmidrule{1-12}
\textbf{Shortened LLaMA}  & 14.54 & 18.44 & 22.82 & 18.60 & 49.67 && 54.14 & 47.15 & 152.18 & 84.49 & 42.05 \\
+ Prune\&Comp & 29.41 & 40.46 & 44.26 & 38.04 & 42.60 && 182.71 & 119.57 & 253.54 & 185.27 & 41.22 \\
+ ReplaceMe (LS) & 70.03 & 55.41 & 134.37 & 86.60 & 43.84 && 1176.61 & 444.62 & 1827.31 & 1149.51 & 38.85 \\
+ ReplaceMe (Cos) & 14.40 & 18.35 & 22.68 & 18.48 & 49.79 && 46.80 & 38.71 & 58.73 & 48.08 & 42.05 \\
+ LinearPatch (Diag)  & 12.45 & 17.50 & 20.35 & 16.77 & 41.38 && 224.29 & 136.05 & 314.63 & 224.99 & 40.84 \\
+ LinearPatch (Rotate)  & 12.46 & 17.17 & 20.27 & 16.63 & 43.75 && 232.96 & 180.62 & 343.98 & 252.52 & 41.41 \\
\jygl + Ghosted Layers (Ours)  & \jygl 10.79 & \jygl 15.26 & \jygl 18.26 & \jygl 14.77 & \jygl 53.55 & \jygl & \jygl 78.81 & \jygl 64.56 & \jygl 196.16 & \jygl 113.18 & \jygl 42.31 \\
\cmidrule{1-12}
\textbf{LLM-Streamline}    & 2301.46 & 1173.53 & 3720.23 & 2398.41 & 42.60 && 4799.22 & 6510.32 & 6173.75 & 5827.76 & 44.68 \\
+ Prune\&Comp & 157.93 & 191.65 & 207.59 & 185.72 & 47.25 && 543.58 & 404.81 & 841.99 & 596.79 & 43.00 \\
+ ReplaceMe (LS) & 29.68 & 26.27 & 51.11 & 35.69 & 59.68 && 133.21 & 78.68 & 400.33 & 204.07 & 51.58 \\
+ ReplaceMe (Cos) & 147.97 & 131.94 & 195.69 & 158.53 & 49.75 && 1469.20 & 1487.28 & 1817.19 & 1591.22 & 46.63 \\
+ LinearPatch (Diag) & 110.93 & 134.68 & 135.50 & 127.04 & 50.46 && 224.32 & 195.13 & 391.71 & 270.39 & 52.39 \\
+ LinearPatch (Rotate) & 57.33 & 73.48 & 67.38 & 66.06 & 53.27 && 232.96 & 180.62 & 343.98 & 252.52 & 52.39 \\
\jygl + Ghosted Layers (Ours) & \jygl 21.35 & \jygl 21.52 & \jygl 40.56 & \jygl 27.81 & \jygl 60.01 & \jygl & \jygl 54.98 & \jygl 41.31 & \jygl 125.73 & \jygl 74.01 & \jygl 53.92 \\
\cmidrule{1-12}
\textbf{ShortGPT} & 63.42 & 69.64 & 68.58 & 67.21 & 57.10 && 3799.22 & 2510.32 & 4173.75 & 3494.43 & 44.71 \\
+ Prune\&Comp & 29.41 & 40.46 & 44.26 & 38.04 & 55.91 && 669.00 & 505.01 & 998.22 & 724.08 & 40.85 \\
+ ReplaceMe (LS) & 519.94 & 318.59 & 137.12 & 325.22 & 36.47 && 1253.00 & 592.00 & 286.00 & 710.33 & 37.48 \\
+ ReplaceMe (Cos) & 64.41 & 67.57 & 103.02 & 78.33 & 56.43 && 423.14 & 538.34 & 425.84 & 462.44 & 39.55 \\
+ LinearPatch (Diag) & 27.03 & 36.95 & 34.14 & 32.71 & 59.12 && 208.68 & 182.72 & 372.63 & 254.68 & 52.58 \\
+ LinearPatch (Rotate)   & 33.16 & 39.51 & 39.27 & 37.31 & 58.90 && 431.44 & 351.86 & 741.12 & 508.14 & 50.93 \\
\jygl + Ghosted Layers (Ours)  & \jygl 15.72 & \jygl 19.73 & \jygl 27.01 & \jygl 20.82 & \jygl 60.45 & \jygl & \jygl 61.14 & \jygl 50.74 & \jygl 116.81 & \jygl 76.23 & \jygl 54.02 \\
\bottomrule
\end{tabular}
\end{adjustbox}
\label{tab:ppl_main}
\end{table*}

\label{sec:settings}

\textbf{Benchmarks.}
We evaluate the performance of \textit{Ghosted Layers} on nine zero-shot commonsense reasoning benchmarks:
\texttt{ARC-E} and \texttt{ARC-C}~\citep{clark2018arc},
\texttt{HellaSwag}~\citep{zellers2019hellaswag},
\texttt{WinoGrande}~\citep{sakaguchi2021winogrande},
\texttt{BoolQ}~\citep{clark2019boolq},
\texttt{OpenbookQA}~\citep{mihaylov2018openbookqa},
\texttt{RTE}~\citep{rte},
\texttt{COPA}~\citep{roemmele2011copa},
and \texttt{RACE}~\citep{lai2017race},
using the \texttt{lm-evaluation-harness} framework~\citep{gao2024lmeval}.
For perplexity, we evaluate on \texttt{WikiText-2}~\citep{wiki},
\texttt{C4}~\citep{raffel2020c4}, and \texttt{Penn Treebank}~\citep{marcus1993ptb}
using non-overlapping windows of length $T = 2{,}048$.

\textbf{Models.}
Our main QA experiments evaluate \texttt{Ghosted Layers} on three open-source LLMs:
LLaMA-3.1-8B~\citep{dubey2024llama3},
DeepSeek-R1-Distill-LLaMA-8B~\citep{guo2025deepseekr1}, and Qwen3-14B~\citep{qwen3}.
All experiments are conducted on a single NVIDIA A40 48GB GPU.

\textbf{Layer pruning and recovery methods.}
We use three layer selection criteria: 
\texttt{LLM-Streamline}~\citep{chen2024streamline}, 
\texttt{ShortGPT}~\citep{men2024shortgpt}, 
and \texttt{Shortened LLaMA}~\citep{kim2024shortened}, 
all using their official implementations. 
For recovery, we compare the following training-free methods: 
\texttt{Prune\&Comp}~\citep{chen2025prunecomp}, 
\texttt{LinearPatch (Diag/Rotate)}~\citep{lp}, 
\texttt{ReplaceMe (LS/Cos)}~\citep{shopkhoev2025replaceme}, 
and \texttt{Ghosted Layers} (ours). 
All baselines use official implementations, except \texttt{Prune\&Comp}, which we reimplement from the paper as no public code is available. 
Our implementation builds on the \texttt{LinearPatch} codebase. 
For layer selection, we use 128 randomly sampled sequences of length $T = 2{,}048$ from the C4 training split~\citep{raffel2020c4}. For operator estimation, LinearPatch, ReplaceMe, and Prune\&Comp use 128 sequences, while Ghosted Layers uses 32 sequences unless otherwise specified.

\subsection{Numerical results}
\textbf{Results on QA Benchmarks.}
Table~\ref{tab:commonsense_main} reports zero-shot accuracy on nine commonsense QA benchmarks at two pruning depths per backbone: 7/32 and 11/32 for the 8B models, and 13/40 and 15/40 for Qwen3-14B. Recovery methods use LLM-Streamline as the pruning criterion. 
Across both pruning ratios and all three backbones, \texttt{Ghosted Layers} attains the highest or competitive average accuracy among training-free recovery methods. 
Expanded results, including LLaMA-3-8B~\cite{dubey2024llama3}, LLaMA-2-7B~\cite{touvron2023llama2}, and OLMo-2-7B~\cite{olmo2}, are provided in Appendix~\ref{app:llms}.

\textbf{Results on PPL Benchmarks.}
Table~\ref{tab:ppl_main} reports perplexity on WikiText-2, C4, and Penn Treebank for LLaMA-3.1-8B across three pruning criteria; results for three backbones are provided in Table~\ref{tab:ppl_extended}. Among these, LLM-Streamline removes a contiguous block of layers, whereas ShortGPT and Shortened LLaMA can select non-contiguous layers. Some pruned models exhibit markedly elevated perplexity, a behavior also observed in prior work~\cite{lp}. The larger values in our experiments likely reflect the more aggressive pruning ratios considered here. Across the three pruning criteria, \texttt{Ghosted Layers} achieves the lowest average perplexity in most settings, with substantial gains under aggressive LLM-Streamline pruning.

\begin{table}[t]
\centering
\footnotesize
\caption{Operator-structure ablation on LLaMA-3.1-8B with LLM-Streamline pruning. Constrained variants are obtained by projecting the same fitted $\mathbf{M}^{*}$ into each operator structure.}
\begin{tabular}{l l c c c c}
\toprule
\textbf{Pruned} & \textbf{Operator} & \textbf{Wiki}$\downarrow$ & \textbf{C4}$\downarrow$ & \textbf{PTB}$\downarrow$ & \textbf{Avg. Acc.}$\uparrow$ \\
\midrule
\rowcolor{gray!20} 7/32 & \textbf{Full (Ours)} & \textbf{21.35} & \textbf{21.52} & \textbf{40.56} & \textbf{60.01} \\
7/32 & Symmetric only & 68.11 & 46.73 & 161.69 & 57.24 \\
7/32 & Anti-symmetric only & 313.14 & 283.71 & 554.10 & 55.18 \\
7/32 & Diagonal & 501.16 & 390.13 & 617.51 & 43.89 \\
\midrule
\rowcolor{gray!20} 11/32 & \textbf{Full (Ours)} & \textbf{54.98} & \textbf{41.31} & \textbf{125.73} & \textbf{53.92} \\
11/32 & Symmetric only & 175.72 & 120.40 & 377.79 & 51.26 \\
11/32 & Anti-symmetric only & 182.54 & 126.39 & 551.22 & 50.63 \\
11/32 & Diagonal & 1076.74 & 588.03 & 1295.80 & 48.54 \\
\bottomrule
\end{tabular}
\label{tab:operator_ablation}
\end{table}

\begin{table}[t]
\centering
\footnotesize
\caption{Ablation on the number of calibration sequences used to estimate $\mathbf{W}^{*}$, evaluated on LLaMA-3.1-8B with 7 out of 32 layers pruned under the LLM-Streamline criterion. Perplexity ($\downarrow$) is reported on WikiText-2, C4, and PTB; PPL AVG is the mean across the three. Acc AVG ($\uparrow$) denotes the mean zero-shot accuracy across the nine commonsense reasoning benchmarks used in our main experiments.}
\begin{tabular}{lccccc}
\toprule
Num.\ of sequences & WikiText-2 & C4 & PTB & PPL AVG $\downarrow$ & Acc AVG $\uparrow$ \\
\midrule
16   & 23.92 & 23.24 & 43.41 & 30.19 & 59.31 \\
\rowcolor{gray!20}
\textbf{32}   & \textbf{21.35} & \textbf{21.52} & \textbf{40.56} & \textbf{27.81} & \textbf{60.01} \\
64   & 20.23 & 21.52 & 39.94 & 27.23 & 59.99 \\
128  & 19.68 & 21.21 & 36.77 & 25.89 & 60.04 \\
\bottomrule
\end{tabular}
\label{tab:csize_ablation}
\end{table}

\setlength{\tabcolsep}{3.5pt}
\begin{table*}[t]
\centering
\footnotesize
\caption{Inference cost, accuracy, and perplexity on LLaMA-3.1-8B under 7/32 and 11/32 layer pruning. GPU memory is reported as the peak activated tensor footprint during the forward pass measured via \texttt{torch.cuda.max\_memory\_allocated()}, normalized to the dense baseline. Latency is the mean prefill time at sequence length 2{,}048 over 10 runs with 3 warmup iterations on the same input.}
\vspace{3pt}
\begin{tabular}{l r r  r r  r r}
\toprule
\textbf{Method} & \textbf{$L_p/L_t$} & \textbf{GPU} (\%) & \textbf{Latency (ms)} & \textbf{Speedup}\textcolor{red}{$\uparrow$} & \textbf{ACC AVG}\textcolor{red}{$\uparrow$} & \textbf{PPL AVG}\textcolor{red}{$\downarrow$} \\
\midrule
Dense                        & 0/32  & 100.0 & $362.6$ & $1.00\times$ & 67.77 & 8.50     \\
\midrule
LLM-Streamline                   & 7/32  & 81.6  & $287.3$ & $1.26\times$ & 42.60 & 2398.41  \\
+ Prune and Comp             & 7/32  & 81.6  & $288.2$ & $1.26\times$ & 47.25 & 185.72   \\
+ ReplaceMe (LS)             & 7/32  & 81.6  & $288.3$ & $1.26\times$ & 59.68 & 35.69    \\
+ ReplaceMe (Cos)            & 7/32  & 81.6  & $287.9$ & $1.26\times$ & 49.75 & 158.53   \\
+ Linear Patch (D)           & 7/32  & 82.0  & $291.4$ & $1.24\times$ & 50.46 & 127.04   \\
+ Linear Patch (R)           & 7/32  & 82.0  & $291.7$ & $1.24\times$ & 53.27 & 66.06    \\
\rowcolor{gray!15}
\jygl + Ghost Layer (Ours)   & \jygl 7/32 & \jygl 82.0 & \jygl $291.9$ & \jygl $1.24\times$ & \jygl 60.01 & \jygl 27.81  \\
\midrule
LLM-Streamline                   & 11/32 & 71.1  & $244.4$ & $1.48\times$ & 44.68 & 5827.76 \\
+ Prune and Comp             & 11/32 & 71.1  & $244.9$ & $1.48\times$ & 43.00 & 596.79   \\
+ ReplaceMe (LS)             & 11/32 & 71.1  & $244.8$ & $1.48\times$ & 51.58 & 204.07   \\
+ ReplaceMe (Cos)            & 11/32 & 71.1  & $245.2$ & $1.48\times$ & 46.63 & 1591.22  \\
+ Linear Patch (D)           & 11/32 & 71.5  & $248.1$ & $1.46\times$ & 52.39 & 270.39   \\
+ Linear Patch (R)           & 11/32 & 71.5  & $248.1$ & $1.46\times$ & 52.39 & 252.52   \\
\rowcolor{gray!15}
\jygl + Ghost Layer (Ours)   & \jygl 11/32 & \jygl 71.5 & \jygl $247.6$ & \jygl $1.46\times$ & \jygl 53.92 & \jygl 74.01 \\
\bottomrule
\end{tabular}

\label{tab:efficiency_llama31}
\end{table*}

\textbf{Ablation on Operator.}
Table~\ref{tab:operator_ablation} compares the full operator with its symmetric-only, anti-symmetric-only, and diagonal variants. The full operator achieves the best accuracy and lowest perplexity at both pruning depths, demonstrating the importance of retaining the unconstrained operator structure for downstream recovery.

\textbf{Ablation on Calibration Set Size.}
We evaluate the effect of calibration set size using $16$ to $128$ C4 sequences of length $T{=}2{,}048$ on LLaMA-3.1-8B with 7 layers pruned by LLM-Streamline. As shown in Table~\ref{tab:csize_ablation}, $32$ sequences achieve $60.01$ average accuracy, with negligible improvements from larger calibration sets. Although perplexity continues to decrease with more sequences, we adopt $32$ as the default to balance recovery performance and calibration cost.

%% file: contents/5-discussion.tex
\begin{wrapfigure}{r}{0.45\textwidth}
    \centering
    \vspace{-0.5cm}
    \includegraphics[width=\linewidth]{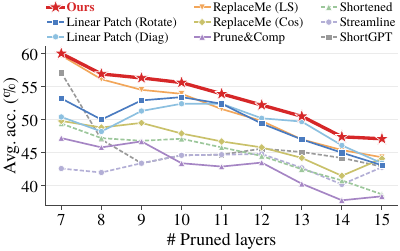}
    \vspace{-0.5cm}
      \caption{Average accuracy across 9 commonsense reasoning benchmarks with LLaMA-3.1-8B}
      \vspace{-0.3cm}
      \label{fig:diss}
\end{wrapfigure}

\textbf{Effect of Pruning Depth on Accuracy.}
Figure~\ref{fig:diss} shows the average commonsense accuracy on LLaMA-3.1-8B as the number of pruned layers increases from 7 to 15. Ghosted Layers consistently achieves the highest accuracy across all pruning depths. Notably, the performance gap over competing methods widens as pruning becomes more aggressive, indicating that the unconstrained operator remains effective even under larger activation mismatch.

\textbf{Efficiency Comparison.}
Table~\ref{tab:efficiency_llama31} reports GPU memory, prefill latency, accuracy, and perplexity for LLaMA-3.1-8B under 7/32 and 11/32 layer pruning (sequence length $2{,}048$, batch size $1$). Latency is measured as the mean prefill time over 10 runs with 3 warmup iterations, and GPU memory is reported as the peak activated tensor footprint. \texttt{LinearPatch} and \texttt{Ghosted Layers} incur identical cost as a single $C \times C$ operation. Under this matched cost, \texttt{Ghosted Layers} consistently achieves higher accuracy and lower perplexity than \texttt{LinearPatch}.

%% file: contents/6-limitation.tex
Ghosted Layers requires unlabeled calibration data and access to the original unpruned model to collect boundary activations $\mathbf{X}_{\mathrm{pre}}$ and $\mathbf{X}_{\mathrm{post}}$. We use 128 C4 sequences for layer selection and 32 sequences for Ghosted Layers operator estimation, while LinearPatch~\cite{lp}, ReplaceMe~\cite{shopkhoev2025replaceme}, and Prune\&Comp use 128 sequences for recovery. Since pruning already requires access to the unpruned model, boundary activations can be collected during the pruning process. Estimating the recovery operator requires a one-time offline regularized least-squares solve, which does not affect inference latency. However, each contiguous pruned region introduces an additional $C \times C$ operator, incurring modest storage and computational overhead.

\textbf{Broader impact}. A discussion of the broader impact of this work is provided in Appendix~\ref{app:BI}.

%% file: contents/6-conclusion.tex
We presented Ghosted Layers, a training-free recovery module for layer-pruned large language models that addresses boundary activation mismatch. Our approach derives a closed-form optimal linear operator from a small calibration set to directly reconstruct the activation discrepancy introduced by pruning. This solution corresponds to the unconstrained optimum of the alignment objective, while LinearPatch is restricted to a constrained operator class. Across multiple LLM backbones and pruning strategies, Ghosted Layers improves accuracy and perplexity over prior training-free recovery methods in most settings at inference cost comparable to LinearPatch. These results demonstrate that solving boundary activation alignment at its unconstrained optimum can substantially recover the performance of layer-pruned models while preserving the efficiency benefits of layer pruning.

%% file: contents/appendix.tex
\setcounter{table}{0}
\renewcommand{\thetable}{A\arabic{table}}

\setcounter{figure}{0}
\renewcommand{\thefigure}{A\arabic{figure}}

\setcounter{equation}{0}
\renewcommand{\theequation}{A\arabic{equation}}

\label{appendix:a}
\section*{Appendix: Ghosted Layers}
This appendix provides supplementary materials that complement the main paper. It includes the proof of our main theoretical result, detailed experimental settings, and additional quantitative results across calibration sizes, model architectures, fine-tuning protocols, and calibration datasets. The appendix is organized as follows:
\begin{itemize}
    \item \textbf{Appendix}~\ref{app:BI}: Broader impact
    \item \textbf{Appendix}~\ref{app:analyses}: Proof for Theorem~\ref{thm:main_main} and Constrained Solution space analysis
    \item \textbf{Appendix}~\ref{app:settings}: Experimental setups
    \item \textbf{Appendix}~\ref{app:finetuning}: Extended results (fine-tuning)
    \item \textbf{Appendix}~\ref{app:llms}: Extended results (Additional LLM architectures)
    \item \textbf{Appendix}~\ref{app:wiki}: Extended results (Alternative calibration datasets)
    \item \textbf{Appendix}~\ref{app:efficiency}: Efficiency measurement
    \item \textbf{Appendix}~\ref{app:solver_comparison}: Closed-form solution computation
\end{itemize}

\section{Broader impact}
\label{app:BI}
This work studies training-free recovery for layer-pruned large language models, aiming to improve the practical deployment efficiency of LLMs while preserving model quality after pruning. By reducing the performance degradation caused by layer removal, our method can help enable smaller and more efficient language models that require less memory and lower inference cost, which may improve accessibility in resource-constrained environments.
At the same time, improved efficiency and deployability of LLMs may also accelerate their large-scale use, potentially amplifying existing societal risks associated with language models, including misinformation, harmful content generation, and misuse at scale. In addition, although layer pruning reduces the computational cost of individual inference instances, broader adoption of compressed LLMs may still increase total energy consumption and environmental impact overall.
Our method is designed as a post-pruning recovery module and does not directly address issues such as fairness, bias, safety, or robustness to malicious use. Future work should therefore study how compression and recovery methods interact with these broader concerns and evaluate compressed models under responsible deployment settings.

\section{Constrained solution space analysis}
\label{app:analyses}
\input{contents/appendix/a-analysis}

\section{Experimental setups}
\label{app:settings}
\input{contents/appendix/a-exp-setting}


\section{Fine-tuning results}
\label{app:finetuning}
\input{contents/appendix/a-exp-finetuning}

\newpage

\section{Additional Large Language Model experiments}
\label{app:llms}

\input{contents/appendix/a-exp-llms}

\section{Ablation study for different calibration dataset}
\label{app:wiki}
\input{contents/appendix/a-exp-wiki}

\newpage
\clearpage

\section{Efficiency measurement}
\input{contents/appendix/a-eff}

\newpage
\clearpage

\section{Closed-form Solution Computation}
\input{contents/appendix/a-svd-computation}

%% file: contents/appendix/a-analysis.tex
\subsection{Proof of Theorem~\ref{thm:main}}
\label{app:proof}
This appendix provides the proof of Theorem~\ref{thm:main_main}, which establishes that $\mathbf{W}^{*} = \mathbf{I} + \mathbf{M}^{*}$ is the minimum-norm solution to the unconstrained activation alignment objective, and that \texttt{LinearPatch}'s symmetric parameterization constitutes a strict subspace restriction of this solution. For completeness, we restate the theorem below before presenting its proof.

\setcounter{theorem}{0}

\begin{theorem}[Ghosted Layers is the Unconstrained Optimum of LinearPatch]
\label{thm:main}
Let $\mathbf{X}_{\mathrm{pre}}, \mathbf{X}_{\mathrm{post}} \in \mathbb{R}^{T_{\mathcal{D}} \times C}$ and $\boldsymbol{\Delta} = \mathbf{X}_{\mathrm{post}} - \mathbf{X}_{\mathrm{pre}}$. Both LinearPatch and Ghosted Layers produce a repaired activation of the form $\mathbf{X}_{\mathrm{new}} = \mathbf{X}_{\mathrm{pre}} \mathbf{W}$, but differ in the structure of $\mathbf{W}$:
\begin{align}
    \mathbf{X}_{\mathrm{new, LP}} = \mathbf{X}_{\mathrm{pre}} \cdot 
    \underbrace{\mathbf{H}\mathbf{D}\mathbf{H}^{\top}}_{\mathbf{W}_{\mathrm{LP}},\ \mathbf{W}_{\mathrm{LP}}^{\top} = \mathbf{W}_{\mathrm{LP}}} \quad 
    & \mathbf{X}_{\mathrm{new, GL}} = \mathbf{X}_{\mathrm{pre}} \cdot 
    \underbrace{(\mathbf{I} + \mathbf{M}^{*})}_{\mathbf{W}^{*},\ \mathbf{W}^{*} \in \mathbb{R}^{C \times C}}
\end{align}
where $\mathbf{M}^{*} = \mathbf{X}_{\mathrm{pre}}^{\dagger}\boldsymbol{\Delta}$, and $\mathbf{W}^{*} = \mathbf{I} + \mathbf{M}^{*}$ is the minimum-norm solution to the unconstrained alignment objective $\min_{\mathbf{W}} \|\mathbf{X}_{\mathrm{pre}}\mathbf{W} - \mathbf{X}_{\mathrm{post}}\|_F^2$, whereas $\mathbf{W}_{\mathrm{LP}}$ searches only over symmetric matrices, making LinearPatch a constrained approximation to Ghosted Layers.
\end{theorem}

\begin{proof}
Both methods produce $\mathbf{X}_{\mathrm{new}} = \mathbf{X}_{\mathrm{pre}}\mathbf{W}$.

Substituting $\mathbf{W} = \mathbf{I} + \mathbf{M}$ into the alignment objective:
\begin{align}
    \|\mathbf{X}_{\mathrm{pre}}\mathbf{W} - \mathbf{X}_{\mathrm{post}}\|_F^2
    &= \|\mathbf{X}_{\mathrm{pre}}(\mathbf{I} + \mathbf{M}) - \mathbf{X}_{\mathrm{post}}\|_F^2 \notag \\
    &= \|\mathbf{X}_{\mathrm{pre}}\mathbf{M} - 
    \underbrace{(\mathbf{X}_{\mathrm{post}} - \mathbf{X}_{\mathrm{pre}})}_{\boldsymbol{\Delta}}\|_F^2 \notag \\
    &= \|\mathbf{X}_{\mathrm{pre}}\mathbf{M} - \boldsymbol{\Delta}\|_F^2.
\end{align}
Hence minimizing over $\mathbf{W}$ is equivalent to $\min_{\mathbf{M}} \|\mathbf{X}_{\mathrm{pre}}\mathbf{M} - \boldsymbol{\Delta}\|_F^2$.

Since the Frobenius norm decouples across columns, it suffices to solve independently for each column $j = 1, \ldots, C$:
\begin{equation}
    \min_{\mathbf{m}_j \in \mathbb{R}^{C}} 
    \|\mathbf{X}_{\mathrm{pre}}\mathbf{m}_j - \boldsymbol{\Delta}_{:,j}\|_2^2.
    \label{eq:col_ls}
\end{equation}
Taking the gradient with respect to $\mathbf{m}_j$ and setting it to zero:
\begin{align}
    \nabla_{\mathbf{m}_j}\|\mathbf{X}_{\mathrm{pre}}\mathbf{m}_j - \boldsymbol{\Delta}_{:,j}\|_2^2
    &= 2\mathbf{X}_{\mathrm{pre}}^{\top}(\mathbf{X}_{\mathrm{pre}}\mathbf{m}_j - \boldsymbol{\Delta}_{:,j}) 
    = \mathbf{0},
\end{align}
which yields the normal equations:
\begin{equation}
    \mathbf{X}_{\mathrm{pre}}^{\top} \mathbf{X}_{\mathrm{pre}}\, \mathbf{m}_j 
    = \mathbf{X}_{\mathrm{pre}}^{\top} \boldsymbol{\Delta}_{:,j}.
\end{equation}
When $\mathbf{X}_{\mathrm{pre}}$ has full column rank, $\mathbf{X}_{\mathrm{pre}}^{\top}\mathbf{X}_{\mathrm{pre}}$ is invertible. To handle the general rank-deficient case, let $\mathbf{X}_{\mathrm{pre}} = \mathbf{U}\boldsymbol{\Sigma}\mathbf{V}^{\top}$ be the thin SVD with $\mathbf{U} \in \mathbb{R}^{T_{\mathcal{D}} \times r}$, $\boldsymbol{\Sigma} = \mathrm{diag}(\sigma_1, \ldots, \sigma_r)$, $\mathbf{V} \in \mathbb{R}^{C \times r}$, and $r = \mathrm{rank}(\mathbf{X}_{\mathrm{pre}})$. The minimum-norm least-squares solution is then:
\begin{align}
    \mathbf{m}_j^{*} 
    &= \left(\mathbf{X}_{\mathrm{pre}}^{\top} \mathbf{X}_{\mathrm{pre}}\right)^{-1}
    \mathbf{X}_{\mathrm{pre}}^{\top} \boldsymbol{\Delta}_{:,j} \\
    &= \mathbf{V}\boldsymbol{\Sigma}^{\dagger}\mathbf{U}^{\top} \boldsymbol{\Delta}_{:,j} \\
    &= \mathbf{X}_{\mathrm{pre}}^{\dagger} \boldsymbol{\Delta}_{:,j},
\end{align}
where the Moore--Penrose pseudoinverse $\mathbf{X}_{\mathrm{pre}}^{\dagger} = \mathbf{V}\boldsymbol{\Sigma}^{\dagger}\mathbf{U}^{\top}$ with:
\begin{equation}
    \left(\boldsymbol{\Sigma}^{\dagger}\right)_{ii} = 
    \begin{cases} 
        1/\sigma_i & \text{if } \sigma_i > \epsilon\,\sigma_1, \\ 
        0 & \text{otherwise,}
    \end{cases}
    \qquad \epsilon = 10^{-6}.
\end{equation}
Stacking the per-column solutions over all $j = 1, \ldots, C$:
\begin{equation}
    \mathbf{M}^{*} 
    = \mathbf{X}_{\mathrm{pre}}^{\dagger}\boldsymbol{\Delta} 
    = \mathbf{V}\boldsymbol{\Sigma}^{\dagger}\mathbf{U}^{\top}\boldsymbol{\Delta},
    \qquad 
    \mathbf{W}^{*} = \mathbf{I} + \mathbf{M}^{*}.
\end{equation}

LinearPatch parameterizes $\mathbf{W}_{\mathrm{LP}} = \mathbf{H}\mathbf{D}\mathbf{H}^{\top}$, where $\mathbf{H} \in \mathbb{R}^{C \times C}$ is the Walsh--Hadamard matrix satisfying $\mathbf{H}^{\top} = \mathbf{H}^{-1}$, and $\mathbf{D} = \mathrm{diag}(d_1,\ldots,d_C)$ is a diagonal matrix.

We verify symmetry by direct computation:
\begin{align}
    \mathbf{W}_{\mathrm{LP}}^{\top} 
    &= \left(\mathbf{H}\mathbf{D}\mathbf{H}^{\top}\right)^{\top} \notag \\
    &= \left(\mathbf{H}^{\top}\right)^{\top} \mathbf{D}^{\top} \mathbf{H}^{\top} \notag \\
    &= \mathbf{H} \mathbf{D}^{\top} \mathbf{H}^{\top}.
\end{align}
Since $\mathbf{D}$ is diagonal, $\mathbf{D}^{\top} = \mathbf{D}$, and therefore:
\begin{equation}
    \mathbf{W}_{\mathrm{LP}}^{\top} = \mathbf{H}\mathbf{D}\mathbf{H}^{\top} = \mathbf{W}_{\mathrm{LP}}.
\end{equation}
Hence $\mathbf{W}_{\mathrm{LP}}$ is symmetric for any choice of $\mathbf{D}$, and lies in the space of $C \times C$ symmetric matrices, a subspace of dimension $\frac{C(C+1)}{2}$, strictly smaller than $C^2 = \dim(\mathbb{R}^{C \times C})$. Therefore, $\mathbf{W}_{\mathrm{LP}}$ cannot attain $\mathbf{W}^{*} = \mathbf{I} + \mathbf{M}^{*}$ whenever $\mathbf{W}^{*}$ has a non-zero anti-symmetric component, i.e., whenever $\mathbf{W}^{*} \neq (\mathbf{W}^{*})^{\top}$.
\end{proof}

\subsection{Computing the Symmetric and Anti-symmetric Decomposition}
\label{app:sym_decomp}

This section details the procedure used to compute the symmetric and anti-symmetric components of $\mathbf{M}^{*}$ reported in Figure~\ref{fig:symm} in Section~\ref{sec:analysis}.

\subsubsection{Experimental Setups}

\paragraph{Models.} We evaluate two open-source LLM backbones: LLaMA-3.1-8B~\citep{dubey2024llama3} and DeepSeek-R1-Distill-LLaMA-8B~\citep{guo2025deepseekr1}. Both models share the same hidden dimension $C{=}4{,}096$.

\paragraph{Calibration data.}
We use 128 sequences of length $T{=}2{,}048$ sampled from the C4 training split~\citep{raffel2020c4} for layer selection, following LinearPatch~\citep{lp}. For the symmetry decomposition, we use $32$ of these sequences to form the boundary activation matrices, yielding $T_{\mathcal{D}} = 32 \times 2{,}048 = 65{,}536$ tokens per backbone.

\paragraph{Pruning criterion.} All backbones use the LLM-Streamline criterion~\citep{chen2024streamline}, which selects the contiguous block $\mathcal{B} = \{\ell^{*}, \ldots, \ell^{*}+n-1\}$ of $n{=}7$ layers whose boundary activations exhibit the highest cosine similarity. We use 128 calibration sequences for block selection and 32 of these sequences for operator computation.

\subsubsection{Procedure}
\paragraph{Step 1: Layer selection.}
We select the pruning block $\mathcal{B}$ using the criterion described above. For each backbone, this yields a specific start index $\ell^{*}$ and end index $\ell^{*}+n$ determined by the cosine-similarity ranking of boundary activations.

\paragraph{Step 2: Boundary activation capture.}
With the unpruned model $\mathcal{M}$ in evaluation mode and \texttt{use\_cache=False}, we register a forward pre-hook on layer $\ell^{*}$ to capture its input $\mathbf{X}_{\mathrm{pre}}$, and a forward pre-hook on layer $\ell^{*}+n$ to capture its input $\mathbf{X}_{\mathrm{post}}$. A forward pre-hook fires immediately before a layer's computation, so the captured tensors correspond exactly to the boundary activations defined in Eq.~\ref{eq:delta}. Hooks are detached from the computation graph and stored on CPU. The collected tensors have shape $\mathbb{R}^{T_{\mathcal{D}} \times C}$.

\paragraph{Step 3: Solving for $\mathbf{M}^{*}$.}
We promote $\mathbf{X}_{\mathrm{pre}}, \mathbf{X}_{\mathrm{post}}$ to \texttt{float64} and form $\boldsymbol{\Delta} = \mathbf{X}_{\mathrm{post}} - \mathbf{X}_{\mathrm{pre}}$. We then solve the regularized normal equations
\begin{equation}
    (\mathbf{X}_{\mathrm{pre}}^{\top}\mathbf{X}_{\mathrm{pre}} + \epsilon \mathbf{I})\,\mathbf{M}_{\epsilon} = \mathbf{X}_{\mathrm{pre}}^{\top} \boldsymbol{\Delta}
    \label{eq:regularized_normal}
\end{equation}
via \texttt{torch.linalg.solve}, which internally uses an LU factorization. The resulting $\mathbf{M}_{\epsilon}$ approximates $\mathbf{M}^{*} = \mathbf{X}_{\mathrm{pre}}^{\dagger} \boldsymbol{\Delta}$ (Eq.~\ref{eq:mstar_method}) and converges to it as $\epsilon \to 0$.

\paragraph{Step 4: Decomposition and reporting.}
We decompose the resulting $\mathbf{M}_{\epsilon} \in \mathbb{R}^{C \times C}$ into its symmetric and anti-symmetric parts using Eq.~\ref{eq:decomp}, compute the Frobenius norms $\|\mathbf{M}_{\epsilon}\|_F$, $\|\mathbf{M}_{\epsilon,\mathrm{sym}}\|_F$, and $\|\mathbf{M}_{\epsilon,\mathrm{asym}}\|_F$, and report the squared norm ratios $\|\mathbf{M}_{\epsilon,\mathrm{sym}}\|_F^2 / \|\mathbf{M}_{\epsilon}\|_F^2$ and $\|\mathbf{M}_{\epsilon,\mathrm{asym}}\|_F^2 / \|\mathbf{M}_{\epsilon}\|_F^2$ in Figure~\ref{fig:symm}. The procedure is identical across both backbones; no model-specific tuning is performed.

%% file: contents/appendix/a-exp-setting.tex
\subsection{Details on pruned models}
\label{app:pruned_models}

All experiments are conducted on officially released LLM checkpoints obtained from Hugging Face, summarized in Table~\ref{tab:model_links}.

\begin{table}[h]
\centering
\footnotesize
\caption{Hugging Face sources for the LLM checkpoints used in our experiments.}
\label{tab:model_links}
\begin{adjustbox}{width=\textwidth}
\begin{tabular}{ll}
\toprule
Model & Download link \\
\midrule
LLaMA-2-7B                 & \url{https://huggingface.co/meta-llama/Llama-2-7b-hf} \\
LLaMA-3-8B                 & \url{https://huggingface.co/meta-llama/Meta-Llama-3-8B} \\
LLaMA-3.1-8B               & \url{https://huggingface.co/meta-llama/Llama-3.1-8B} \\
OLMo-2-7B                  & \url{https://huggingface.co/allenai/OLMo-2-1124-7B} \\
Qwen-3-14B                 & \url{https://huggingface.co/Qwen/Qwen3-14B} \\
DeepSeek-R1-Distill-Llama-8B & \url{https://huggingface.co/deepseek-ai/DeepSeek-R1-Distill-Llama-8B} \\
\bottomrule
\end{tabular}
\end{adjustbox}
\end{table}

\subsection{Details on pruning and recovery methods}
\label{app:methods}

This section describes the pruning criteria and recovery methods used as baselines in our experiments. All baselines are reproduced from their official implementations to ensure a fair comparison. Table~\ref{tab:method_repos} lists the official repositories for each method.

\begin{table}[h]
\centering
\small
\caption{Official repositories of the pruning criteria and recovery methods used in our experiments.}
\label{tab:method_repos}
\begin{adjustbox}{width=\textwidth}
\begin{tabular}{lll}
\toprule
Category & Method & Official repository \\
\midrule
\multirow{3}{*}{Pruning criterion}
& ShortGPT~\citep{men2024shortgpt}       & \url{https://github.com/sramshetty/ShortGPT} \\
& Shortened LLaMA~\citep{kim2024shortened} & \url{https://github.com/Nota-NetsPresso/shortened-llm} \\
& LLM-Streamline~\citep{chen2024streamline} & \url{https://github.com/ruckbreasoning/llm-streamline} \\
\midrule
\multirow{4}{*}{Recovery method}
& Prune\&Comp~\citep{chen2025prunecomp}    & N/A \\
& ReplaceMe~\citep{shopkhoev2025replaceme} & \url{https://github.com/mts-ai/ReplaceMe} \\
& LinearPatch~\citep{lp}  & \url{https://github.com/chenxinrui-tsinghua/LinearPatch} \\
& \texttt{Ghosted Layers} (Ours)           & \url{https://github.com/daniel-eai/ghosted_layers_official_repository/} \\
\bottomrule
\end{tabular}
\end{adjustbox}
\end{table}

\subsubsection{Pruning criteria}

\paragraph{ShortGPT~\citep{men2024shortgpt}.}
ShortGPT assigns each layer a Block Influence (BI) score defined as one minus the cosine similarity between its input and output hidden states, averaged over a calibration set. Layers with the lowest BI scores are removed in a one-shot manner. We reproduce ShortGPT using its official implementation, and BI scores are computed on the same 128 C4 sequences used throughout the paper.

\paragraph{Shortened LLaMA~\citep{kim2024shortened}.}
Shortened LLaMA evaluates each layer's contribution by the perplexity degradation incurred when that layer is removed from the original model, and prunes the layers with the smallest perplexity impact. We use the official implementation and the released pruned-layer indices where available.

\paragraph{LLM-Streamline~\citep{chen2024streamline}.}
LLM-Streamline selects a single \emph{contiguous} block of layers to remove, choosing the block $\mathcal{B} = \{\ell^{*}, \ldots, \ell^{*}+n-1\}$ whose boundary activations exhibit the highest cosine similarity. This minimizes the representation change across the pruned region. We reproduce LLM-Streamline using its official implementation and adopt it as the default pruning criterion in our main experiments unless otherwise specified.

\subsubsection{Recovery methods}

\paragraph{Prune\&Comp~\citep{chen2025prunecomp}.}
Prune\&Comp rescales the surviving boundary weights with per-channel scalar factors estimated from the calibration set, compensating for magnitude shifts without introducing any additional parameters or cross-channel interaction. As no official implementation is publicly available at the time of submission, we reimplement Prune\&Comp from the description in the original paper, using 128 C4 calibration sequences for recovery.

\paragraph{ReplaceMe~\citep{shopkhoev2025replaceme}.}
ReplaceMe approximates the computation of the pruned block by a linear transformation applied to the boundary block's MLP output, and absorbs this transformation into the surviving MLP weights. We reproduce ReplaceMe using its official implementation and include two variants:
\begin{itemize}
    \item \textbf{ReplaceMe (LS)}, which estimates the linear map via least squares
    \item \textbf{ReplaceMe (Cos)}, which optimizes a cosine-distance objective with Adam for 10 epochs using the default hyperparameters of the official implementation.
\end{itemize}

\paragraph{LinearPatch~\citep{lp}.}
LinearPatch inserts a single matrix multiplication at the pruning boundary, parameterized as $\mathbf{W}_{\mathrm{LP}} = \mathbf{H}\mathbf{D}\mathbf{H}^{\top}$, where $\mathbf{H}$ is the Walsh--Hadamard matrix and $\mathbf{D}$ is a diagonal scaling matrix. This parameterization is real symmetric by construction. We reproduce LinearPatch using its official implementation and include two variants:
\begin{itemize}
    \item \textbf{LinearPatch (Diag)}, which applies only channel-wise scaling,
    \item \textbf{LinearPatch (Rotate)}, which additionally applies the Hadamard rotation.
\end{itemize}

\paragraph{\texttt{Ghosted Layers} (Ours).}
\texttt{Ghosted Layers} inserts an unconstrained linear operator $\mathbf{W}^{*} = \mathbf{I} + \mathbf{M}^{*}$ at the pruning boundary, where $\mathbf{M}^{*} = \mathbf{X}_{\mathrm{pre}}^{\dagger} \boldsymbol{\Delta}$ is the closed-form minimum-norm solution to the alignment objective defined in Section~\ref{sec:method}. Our implementation builds on the official LinearPatch codebase and uses 32 C4 calibration sequences of length $T{=}2{,}048$ for operator estimation across all main experiments.

\subsubsection{Details of evaluation benchmarks}
\label{app:eval}

We assess model quality along two axes: language modeling perplexity and downstream task accuracy.

\paragraph{Perplexity (PPL).}
We report perplexity on three standard language modeling corpora: WikiText-2~\citep{wiki}, C4~\citep{raffel2020c4}, and Penn Treebank (PTB)~\citep{marcus1993ptb}. Perplexity is computed on non-overlapping windows of length $T{=}2{,}048$, consistent with the calibration sequence length.

\paragraph{Commonsense QA.}
For downstream accuracy, we evaluate on nine zero-shot commonsense reasoning benchmarks: ARC-Easy and ARC-Challenge~\citep{clark2018arc}, HellaSwag~\citep{zellers2019hellaswag}, WinoGrande~\citep{sakaguchi2021winogrande}, BoolQ~\citep{clark2019boolq}, OpenbookQA~\citep{mihaylov2018openbookqa}, RTE~\citep{rte}, COPA~\citep{roemmele2011copa}, and RACE~\citep{lai2017race}.

\paragraph{Evaluation framework.}
All perplexity and QA benchmarks are evaluated using the \texttt{lm-evaluation-harness} library from \url{https://github.com/EleutherAI/lm-evaluation-harness}, following the default evaluation protocols.

%% file: contents/appendix/a-exp-finetuning.tex
\subsection{Fine-tuning setup}

We follow the fine-tuning protocol of \texttt{LinearPatch}~\citep{lp} exactly to ensure a fair head-to-head comparison under identical post-training conditions. Specifically, we adopt a memory-efficient offline knowledge distillation strategy where the pruned model (student) is trained to match the output distribution of the unpruned model (teacher) via Kullback--Leibler (KL) divergence on the top-$K$ logits.

\paragraph{Distillation objective.}
We optimize only the boundary operator ($\mathbf{W}^{*}$ for \texttt{Ghosted Layers} or $\mathbf{W}_{\mathrm{LP}}$ for \texttt{LinearPatch}) while freezing all other parameters of the pruned model. For each training sample $\mathbf{x} \in \mathcal{T}$, we minimize:
\begin{equation}
    \min_{\mathbf{W}} \mathbb{E}_{\mathbf{x} \in \mathcal{T}} \, \mathrm{KL}\!\left(\mathbf{o}_{t}(\mathbf{x}),\ \mathbf{o}_{s}(\mathbf{x})\right),
\end{equation}
where $\mathbf{o}_{t}$ and $\mathbf{o}_{s}$ denote the top-$K$ logits probability distributions from the teacher and student, respectively, using the teacher's vocabulary indices for both. Following LinearPatch paper~\cite{lp}, we set $K = 100$.

\paragraph{Training configuration.}
We use the identical configuration across all compared methods (\texttt{LinearPatch (Diag) + FT}, \texttt{LinearPatch (Rotate) + FT}, and \texttt{Ghosted Layers + FT}) to isolate the contribution of the operator parameterization:
\begin{itemize}
    \item \textbf{Optimizer:} AdamW~\citep{loshchilov2019adamw} with a learning rate of $1 \times 10^{-4}$.
    \item \textbf{Training data:} 5{,}000 sequences of length $T = 2{,}048$ randomly sampled from the C4 training split~\citep{raffel2020c4}.
    \item \textbf{Schedule:} One epoch, with no learning rate warmup or decay.
    \item \textbf{Frozen parameters:} only the boundary operator is updated; all other model parameters are frozen.
\end{itemize}
For \texttt{Ghosted Layers}, the closed-form $\mathbf{W}^{*} = \mathbf{I} + \mathbf{M}^{*}$ is used as the initialization and then fine-tuned under the same objective. We drop any structural constraint on $\mathbf{W}^{*}$ during fine-tuning, consistent with the unconstrained formulation in Section~\ref{sec:method}.

\paragraph{Evaluation.}
We evaluate fine-tuned models on the same nine commonsense QA benchmarks (Table~\ref{tab:commonsensekl_common}) and three perplexity corpora, WikiText-2~\citep{wiki}, C4~\citep{raffel2020c4}, and PTB~\citep{marcus1993ptb} (Table~\ref{tab:commonsensekl_ppl}), used throughout the main paper. Perplexity is computed on non-overlapping windows of length $T = 2{,}048$.

\begin{table*}[t]
\centering
\scriptsize
\caption{Comparison on Commonsense Reasoning benchmark with SOTA post-training method LLM-Streamline and Linear Patch 7-layer pruning}
\begin{adjustbox}{width=1\columnwidth}
\begin{tabular}{l l c c c c c c c c c c}
\toprule
\textbf{Model} & \textbf{Method} & \textbf{ARC-E} & \textbf{ARC-C} & \textbf{HellaS} & \textbf{WinoG} & \textbf{BoolQ} & \textbf{OBQA} & \textbf{RTE} & \textbf{CoPa} & \textbf{Race} & \textbf{AVG}\\
\midrule
\multirow{5}{*}{\begin{sideways}LLaMA-2-7B\end{sideways}}
& Dense          & 74.49 & 46.25 & 75.99 & 68.90 & 77.71 & 44.20 & 62.82 & 87.00 & 39.62 & 64.11\\
\cmidrule{2-12}
& Pruned LLM & 55.89 & 36.18 & 62.64 & 66.38 & 62.17 & 37.20 & 52.35 & 81.00 & 33.78 & 54.18\\
& Linear Patch (Diag) + FT & 61.62 & 37.63 & 68.61 & 68.19 & 72.02 & 36.60 & 68.59 & 84.00 & 37.51 & 59.42\\
& Linear Patch (Rotate) + FT & 61.95 & 37.97 & 68.59 & 68.27 & 71.83 & 37.00 & 66.79 & 86.00 & 38.28 & 59.63\\
& \jygl \textbf{Ghost Layer (Ours) + FT} & \jygl62.92 & \jygl36.69 & \jygl67.87 & \jygl67.01 & \jygl75.66 & \jygl37.40 & \jygl67.15 & \jygl85.00 & \jygl37.70 & \jygl59.71\\
\midrule
\multirow{5}{*}{\begin{sideways}LLaMA3-8B\end{sideways}}
& Dense          & 77.65 & 53.41 & 79.16 & 72.38 & 81.35 & 45.00 & 69.68 & 89.00 & 40.00 & 67.51\\
\cmidrule{2-12}
& Pruned LLM & 39.69 & 28.84 & 33.18 & 55.49 & 38.07 & 29.60 & 57.40 & 60.00 & 24.02 & 40.70\\
& Linear Patch (Diag) + FT & 64.44 & 44.11 & 71.17 & 72.38 & 69.51 & 37.00 & 65.70 & 83.00 & 37.42 & 60.53\\
& Linear Patch (Rotate) + FT & 64.60 & 44.28 & 71.09 & 73.16 & 70.12 & 37.00 & 66.06 & 83.00 & 37.22 & 60.73\\
& \jygl \textbf{Ghost Layer (Ours) + FT} & \jygl67.59 & \jygl44.11 & \jygl69.19 & \jygl72.22 & \jygl75.54 & \jygl40.00 & \jygl63.90 & \jygl85.00 & \jygl36.65 & \jygl61.58\\
\midrule
\multirow{5}{*}{\begin{sideways}LLaMA3.1-8B\end{sideways}}
& Dense          & 81.19 & 53.41 & 78.92 & 73.64 & 82.11 & 44.80 & 69.68 & 87.00 & 39.14 & 67.77\\
\cmidrule{2-12}
& Pruned LLM & 44.19 & 33.11 & 33.39 & 56.99 & 38.20 & 32.60 & 58.12 & 61.00 & 25.84 & 42.60\\
& Linear Patch (Diag) + FT & 67.42 & 43.26 & 70.93 & 72.14 & 67.80 & 36.60 & 69.31 & 81.00 & 38.37 & 60.76\\
& Linear Patch (Rotate) + FT & 67.80 & 43.77 & 70.88 & 72.22 & 67.13 & 36.80 & 69.31 & 81.00 & 37.80 & 60.75\\
& \jygl \textbf{Ghost Layer (Ours) + FT} & \jygl70.03 & \jygl43.69 & \jygl68.81 & \jygl71.67 & \jygl73.09 & \jygl40.20 & \jygl67.87 & \jygl84.00 & \jygl36.75 & \jygl61.79\\
\midrule
\multirow{5}{*}{\begin{sideways}D-LLaMA-8B\end{sideways}}
& Dense          & 65.87 & 42.41 & 74.35 & 67.80 & 82.91 & 41.20 & 69.68 & 89.00 & 41.63 & 63.87\\
\cmidrule{2-12}
& Pruned LLM & 49.92 & 35.24 & 46.33 & 61.56 & 51.99 & 32.40 & 57.40 & 70.00 & 27.27 & 48.01\\
& Linear Patch (Diag) + FT & 56.48 & 37.46 & 67.92 & 64.80 & 81.28 & 37.80 & 70.04 & 83.00 & 38.85 & 59.74\\
& Linear Patch (Rotate) + FT & 56.48 & 37.54 & 67.72 & 68.11 & 80.15 & 36.40 & 68.95 & 83.00 & 38.37 & 59.64\\
& \jygl Ghost Layer (Ours) + FT & \jygl57.37 & \jygl38.91 & \jygl66.01 & \jygl66.77 & \jygl77.83 & \jygl37.80 & \jygl72.92 & \jygl87.00 & \jygl39.33 & \jygl60.44\\
\bottomrule
\end{tabular}
\end{adjustbox}
\label{tab:commonsensekl_common}
\end{table*}

\begin{table*}[t]
\centering
\scriptsize
\caption{Comparison on PPL and Commonsense Reasoning benchmark with SOTA post-training method LLM-Streamline and Linear Patch 7-layer pruning}
\begin{tabular}{l l r r r r}
\toprule
\textbf{Model} & \textbf{Method }& \textbf{WIKI} &\textbf{C4} & \textbf{PTB} & \textbf{PPL AVG}\\
\midrule
\multirow{5}{*}{\begin{sideways}LLaMA-2-7B\end{sideways}}
& Dense          & 5.47 & 6.71 & 22.51 & 11.56\\
\cmidrule{2-6}
& Pruned LLM & 18.45 & 25.37 & 62.18 & 35.33\\
& Linear Patch (Diag) + FT&  11.13 & 12.19 & 37.09 & 20.14\\
& Linear Patch (Rotate) + FT &  10.91 & 12.07 & 37.28 & 20.09\\
& \jygl \textbf{Ghost Layer (Ours) + FT} &  \jygl10.01 & \jygl10.31 & \jygl31.27 & \jygl17.20\\
\midrule
\multirow{5}{*}{\begin{sideways}LLaMA3-8B\end{sideways}}
& Dense          &  6.14 & 8.61 & 10.58 & 8.44\\
\cmidrule{2-6}
& Pruned LLM & 2305.48 & 2106.64 & 4642.40 & 3018.17\\
& Linear Patch (Diag)+ FT & 17.29 & 21.59 & 30.25 & 23.04\\
& Linear Patch (Rotate)+ FT &  17.15 & 21.62 & 30.48 & 23.08\\
& \jygl \textbf{Ghost Layer (Ours) + FT} & \jygl13.00 & \jygl14.96 & \jygl23.03 & \jygl17.00\\
\midrule
\multirow{5}{*}{\begin{sideways}LLaMA3.1-8B\end{sideways}}
& Dense          & 6.24 & 8.68 & 10.58 & 8.50\\
\cmidrule{2-6}
& Pruned LLM & 2301.46 & 1173.53 & 3720.23 & 2398.41\\
& Linear Patch (Diag) + FT & 17.26 & 21.51 & 30.20 & 22.99\\
& Linear Patch (Rotate) + FT & 17.05 & 21.54 & 30.33 & 22.97\\
& \jygl \textbf{Ghost Layer (Ours) + FT} & \jygl 13.10 & \jygl 14.99 & \jygl23.14 & \jygl17.08\\
\midrule
\multirow{5}{*}{\begin{sideways}D-LLaMA-8B\end{sideways}}
& Dense          & 13.13 & 19.46 & 22.28 & 18.29\\
\cmidrule{2-6}
& Pruned LLM &  3083.95 & 1291.64 & 2985.68 & 2453.76\\
& Linear Patch (Diag) + FT&  35.96 & 34.63 & 51.01 & 40.53\\
& Linear Patch (Rotate) + FT & 33.64 & 33.47 & 51.39 & 39.50\\
& \jygl Ghost Layer (Ours) + FT &  \jygl22.92 & \jygl25.26 & \jygl33.66 & \jygl27.28\\
\bottomrule
\end{tabular}
\label{tab:commonsensekl_ppl}
\end{table*}

\subsection{Fine-tuning Results}

Tables~\ref{tab:commonsensekl_common} and~\ref{tab:commonsensekl_ppl} report the QA accuracy and perplexity comparisons between \texttt{Ghosted Layers + FT} and \texttt{LinearPatch + FT} across four LLM backbones (LLaMA-2-7B, LLaMA-3-8B, LLaMA-3.1-8B, and DeepSeek-R1-Distill-LLaMA-8B) with 7 layers pruned.

\paragraph{QA accuracy.}
As shown in Table~\ref{tab:commonsensekl_common}, \texttt{Ghosted Layers + FT} attains the highest average accuracy across all four backbones under the same fine-tuning budget. The margin over the stronger \texttt{LinearPatch (Rotate) + FT} ranges from $+0.08$ on LLaMA-2-7B to $+1.04$ on LLaMA-3.1-8B, suggesting that the unconstrained operator continues to benefit from lightweight post-training, despite the fine-tuned \texttt{LinearPatch} variants being free to move out of the symmetric subspace during training.

\paragraph{Perplexity.}
Table~\ref{tab:commonsensekl_ppl} shows a more pronounced gap on the language modeling benchmarks. \texttt{Ghosted Layers + FT} achieves the lowest average perplexity across all four backbones, with relative reductions of $14.4\%$ (LLaMA-2-7B), $26.3\%$ (LLaMA-3-8B), $25.7\%$ (LLaMA-3.1-8B), and $30.9\%$ (DeepSeek-R1-Distill-LLaMA-8B) over \texttt{LinearPatch (Rotate) + FT}. This pattern is consistent with the interpretation that the closed-form $\mathbf{W}^{*}$ provides a lower-error initialization on the boundary alignment objective, and that starting distillation from this initialization yields better generative quality within the same training budget.

%% file: contents/appendix/a-exp-llms.tex
\label{app:additional_model_results}

To evaluate the generality of \texttt{Ghosted Layers}, we report expanded zero-shot QA results for LLaMA-3-8B~\citep{dubey2024llama3}, LLaMA-2-7B~\citep{touvron2023llama2}, and OLMo-2-7B~\citep{olmo2}, alongside detailed results for the main-text backbones. These models span different pretraining corpora, model generations, and scales, allowing us to assess whether the recovery quality of \texttt{Ghosted Layers} transfers across architectural and training variations. All other experimental settings---including the calibration corpus (C4, $T{=}2{,}048$; 32 sequences for Ghosted Layers and 128 for LinearPatch and ReplaceMe), pruning criterion (LLM-Streamline~\citep{chen2024streamline}), and evaluation protocol across nine commonsense QA benchmarks---are identical to the main experiments in Table~\ref{tab:commonsense_main}. For Qwen3-14B, which has $L{=}40$ layers, we report results at $13/40$ and $15/40$ pruning ratios alongside $7/32$ and $11/32$ for the $32$-layer backbones.

Table~\ref{tab:commonsense_additional_llms} reports the results for OLMo-2-7B, LLaMA-2-7B, and Qwen3-14B. \texttt{Ghosted Layers} attains the highest average accuracy in all six settings across the three backbones and two pruning ratios, consistent with the trend observed in the main paper. Notably, the margin over the strongest training-free recovery baseline widens on Qwen3-14B under aggressive pruning ($15/40$), where \texttt{Ghosted Layers} achieves $43.61$ AVG accuracy versus $40.08$ for the next-best recovery method (\texttt{ReplaceMe (Cos)}). These results indicate that \texttt{Ghosted Layers} generalizes favorably across architectures with different pretraining corpora, scales, and design choices, including grouped-query attention variants such as Qwen-3.



\begin{table}[h]
\centering
\footnotesize
\caption{Recovery on LLaVA-1.5-7B with 11/32 language-decoder layers pruned by LLM-Streamline. Higher scores are better. AVG is the arithmetic mean of the three reported benchmark scores.}
\begin{tabular}{l r r r r}
\toprule
Method & MME & SQA & MMMU & AVG \\
\midrule
Dense & 79.85 & 65.27 & 38.72 & 61.28 \\
No recovery & 50.17 & 58.35 & 35.30 & 47.94 \\
Ghosted Layers & 68.52 & 61.68 & 36.13 & 55.44 \\
ReplaceMe (LS) & 66.41 & 60.73 & 33.77 & 53.64 \\
LinearPatch (Diag) & 63.29 & 61.43 & 35.42 & 53.38 \\
\bottomrule
\end{tabular}
\label{tab:llava_recovery}
\end{table}

\paragraph{Vision-language model.}
We additionally evaluate Ghosted Layers on LLaVA-1.5-7B, pruning 11 of its 32 language-decoder layers using LLM-Streamline and following the recovery settings in the main experiments. Table~\ref{tab:llava_recovery} reports the scores on MME, SQA, and MMMU. Ghosted Layers improves the average score from $47.94$ without recovery to $55.44$, exceeding ReplaceMe (LS) and LinearPatch (Diag) by $1.80$ and $2.06$ points, respectively. These results extend the evaluation to a Transformer-based vision-language model.

\begin{table*}[t]
\centering
\scriptsize
\caption{Zero-shot accuracy (\%) on commonsense QA benchmarks for 7-layer and 11-layer pruning across three LLM backbones. All methods are training-free. AVG denotes the mean accuracy across all nine tasks.  $L_p/L_t$ denotes the number of pruned layers $L_p$ over the total number of layers $L_t$ in the original model.}
\begin{adjustbox}{width=1\columnwidth}
\begin{tabular}{l  l  l c c c c c c c c c  r }
\toprule

\textbf{Model} & \textbf{$L_p/L_t$ } & \textbf{Method }&  \textbf{ARC-E}& \textbf{ARC-C }& \textbf{HellaS} & \textbf{WinoG} & \textbf{BoolQ}  & \textbf{OBQA} & \textbf{RTE} & \textbf{COPA} & \textbf{Race} & \textbf{AVG} \textcolor{red}{$\uparrow$} \\
\midrule
\multirow{19}{*}{\begin{sideways}LLaMA-3-8B\end{sideways}}
& 0/32 & Dense          &  77.65 & 53.41 & 79.16 & 72.38 & 81.35 & 45.00 & 69.68 & 89.00 & 40.00 & 67.51 \\
\cmidrule{2-13}
& 7/32 & Shortened LLaMA & 58.84 & 32.68 & 59.16 & 53.75 & 45.38 & 34.40 & 54.15 & 75.00 & 30.72 & 49.34\\
& 7/32 & LLM-Streamline & 39.69 & 28.84 & 33.18 & 55.49 & 38.07 & 29.60 & 57.40 & 60.00 & 24.02 & 40.70 \\
& 7/32 & ShortGPT & 56.65 & 42.41 & 64.69 & 71.35 & 65.14 & 32.80 & 67.87 & 75.00 & 34.16 & 56.67 \\
& 7/32 & Prune\&Comp & 42.97 & 29.69 & 41.64 & 58.88 & 52.84 & 33.80 & 62.09 & 70.00 & 27.08 & 46.55\\
&7/32 & ReplaceMe (LS) & 63.13 & 43.86 & 64.07 & 72.85 & 71.31 & 38.00 & 68.59 & 77.00 & 36.27 & 59.45\\
&7/32 & ReplaceMe (Cos) & 52.10 & 35.67 & 52.38 & 66.69 & 39.24 & 34.60 & 63.90 & 67.00 & 30.05 & 49.07 \\
&7/32 & LinearPatch (Diag) & 43.86 & 31.74 & 44.14 & 60.85 & 61.25 & 33.60 & 65.34 & 68.00 & 29.09 & 48.65\\
&7/32 & LinearPatch (Rotate) & 50.51 & 34.56 & 49.90 & 63.77 & 57.49 & 33.40 & 67.15 & 68.00 & 30.14 & 50.55 \\
&7/32 & \jygl Ghosted Layers (Ours) & \jygl 65.82 & \jygl 43.69 & \jygl 66.55 & \jygl 71.27 & \jygl 74.34 & \jygl 37.20 & \jygl 66.43 & \jygl 79.00 & \jygl 36.56 & \jygl  60.10\\
\cmidrule{2-13}
&11/32& Shortened LLaMA  & 48.65 & 29.27 & 49.81 & 51.78 & 61.65 & 30.60 & 50.90 & 71.00 & 28.23 & 46.88 \\
&11/32& LLM-Streamline & 38.17 & 29.69 & 33.04 & 56.67 & 56.15 & 30.20 & 70.04 & 57.00 & 27.18 & 44.24 \\
&11/32& ShortGPT & 38.17 & 29.69 & 33.04 & 56.67 & 56.15 & 30.20 & 70.04 & 57.00 & 27.18 & 44.24 \\
&11/32& Prune\&Comp &  32.28 & 25.00 & 37.78 & 50.99 & 55.90 & 30.40 & 49.46 & 58.00 & 23.25 & 40.34\\
&11/32& ReplaceMe (LS) &  42.93 & 33.87 & 47.30 & 67.40 & 75.75 & 32.40 & 64.62 & 69.00 & 31.67 & 51.66 \\
&11/32& ReplaceMe (Cos) & 42.68 & 33.45 & 38.27 & 58.56 & 61.90 & 29.20 & 62.09 & 60.00 & 30.05 & 46.24 \\
&11/32& LinearPatch (Diag) & 46.34 & 35.15 & 47.72 & 61.01 & 71.74 & 31.00 & 70.76 & 68.00 & 30.91 & 51.40\\
&11/32& LinearPatch (Rotate) & 48.15 & 34.56 & 45.04 & 61.40 & 75.47 & 31.60 & 66.79 & 68.00 & 31.29 & 51.37  \\
&11/32& \jygl Ghosted Layers (Ours) & \jygl48.23 & \jygl34.56 & \jygl53.35 & \jygl69.77 & \jygl75.57 & \jygl32.40 & \jygl65.70 & \jygl71.00 & \jygl32.34 & \jygl53.66\\
\midrule
\multirow{19}{*}{\begin{sideways}LLaMA-3.1-8B\end{sideways}}
&0/32 & Dense          &  81.19 & 53.41 & 78.92 & 73.64 & 82.11 & 44.80 & 69.68 & 87.00 & 39.14 & 67.77\\
\cmidrule{2-13}
&7/32 & Shortened LLaMA & 61.32 & 33.11 & 59.55 & 54.22 & 43.76 & 35.20 & 51.62 & 77.00 & 31.29 & 49.67\\
&7/32 & LLM-Streamline & 44.19 & 33.11 & 33.39 & 56.99 & 38.20 & 32.60 & 58.12 & 61.00 & 25.84 & 42.60 \\
&7/32 & ShortGPT & 58.29 & 42.15 & 64.96 & 68.35 & 61.99 & 34.40 & 69.68 & 80.00 & 34.45 & 57.14 \\
&7/32 & Prune\&Comp & 46.25 & 30.89 & 44.22 & 59.12 & 53.91 & 35.00 & 60.29 & 67.00 & 28.61 & 47.25 \\
&7/32 & ReplaceMe (LS) &  64.90 & 43.52 & 63.80 & 71.67 & 69.60 & 37.80 & 71.48 & 77.00 & 37.32 & 59.68\\
&7/32 & ReplaceMe (Cos) &  57.03 & 37.29 & 52.14 & 64.25 & 39.36 & 35.80 & 62.82 & 68.00 & 31.10 & 49.75 \\
&7/32 & LinearPatch (Diag) & 48.36 & 32.68 & 47.36 & 61.56 & 63.70 & 35.00 & 68.23 & 68.00 & 29.28 & 50.46 \\
&7/32 & LinearPatch (Rotate) & 57.62 & 37.29 & 54.91 & 65.43 & 60.49 & 35.80 & 69.68 & 69.00 & 29.19 & 53.27 \\
&7/32 & \jygl Ghosted Layers (Ours) & \jygl 68.01 & \jygl 43.00 & \jygl 66.60 & \jygl 71.59 & \jygl 71.31 & \jygl 37.20 & \jygl 68.59 & \jygl 76.00 & \jygl 37.80 & \jygl 60.01  \\
\cmidrule{2-13}
&11/32& Shortened LLaMA  & 38.54 & 30.38 & 29.05 & 50.91 & 59.80 & 28.20 & 50.26 & 62.00 & 29.38 & 42.05 \\
&11/32& LLM-Streamline &  39.69 & 30.29 & 31.49 & 56.12 & 55.11 & 30.40 & 69.68 & 61.00 & 28.33 & 44.68 \\
&11/32& ShortGPT &  39.69 & 30.29 & 31.49 & 56.12 & 55.11 & 30.40 & 69.68 & 61.00 & 28.33 & 44.68 \\
&11/32& Prune\&Comp & 38.51 & 27.30 & 40.50 & 53.28 & 62.08 & 29.40 & 53.79 & 57.00 & 25.17 & 43.00 \\
&11/32& ReplaceMe (LS) &  44.53 & 34.04 & 47.23 & 67.40 & 73.55 & 29.80 & 70.76 & 67.00 & 29.95 & 51.58\\
&11/32& ReplaceMe (Cos) & 43.35 & 32.59 & 37.02 & 56.75 & 58.13 & 29.60 & 68.23 & 62.00 & 31.96 & 46.63 \\
&11/32& LinearPatch (Diag) & 50.67 & 36.60 & 48.55 & 62.51 & 71.31 & 31.00 & 71.84 & 68.00 & 31.00 & 52.39 \\
&11/32& LinearPatch (Rotate) &  50.93 & 34.98 & 47.30 & 62.19 & 74.65 & 30.40 & 71.48 & 68.00 & 31.58 & 52.39\\
&11/32& \jygl Ghosted Layers (Ours) & \jygl50.08 & \jygl34.64 & \jygl52.73 & \jygl69.30 & \jygl72.63 & \jygl31.80 & \jygl69.68 & \jygl72.00 & \jygl32.44 & \jygl53.92\\
\midrule
\multirow{19}{*}{\begin{sideways}DeepSeek-R1-Distill-LLaMA-8B\end{sideways}}
&0/32& Dense & 65.87 & 42.41 & 74.35 & 67.80 & 82.91 & 41.20 & 69.68 & 89.00 & 41.63 & 63.87\\
\cmidrule{2-13}
&7/32& Shortened LLaMA & 45.71 & 29.69 & 55.66 & 57.70 & 64.65 & 30.40 & 50.54 & 76.00 & 30.62 & 49.00\\
&7/32& LLM-Streamline & 49.92 & 35.24 & 46.33 & 61.56 & 51.99 & 32.40 & 57.40 & 70.00 & 27.27 & 48.01\\
&7/32& ShortGPT & 49.12 & 37.03 & 55.95 & 63.06 & 77.09 & 34.00 & 74.73 & 71.00 & 33.21 & 55.02\\
&7/32& Prune\&Comp & 44.99 & 31.06 & 43.44 & 56.12 & 53.12 & 31.80 & 60.65 & 64.00 & 26.79 & 45.77 \\
&7/32& ReplaceMe (LS) & 55.35 & 37.03 & 59.97 & 66.46 & 69.08 & 35.20 & 75.09 & 80.00 & 37.80 & 57.33\\
&7/32& ReplaceMe (Cos) &  51.09 & 36.09 & 53.48 & 64.17 & 58.01 & 36.00 & 58.84 & 74.00 & 33.49 & 51.69 \\
&7/32& LinearPatch (Diag) & 51.09 & 35.32 & 49.59 & 62.27 & 62.78 & 34.60 & 66.43 & 73.00 & 30.05 & 51.68 \\
&7/32& LinearPatch (Rotate) & 54.08 & 36.18 & 53.29 & 64.25 & 72.97 & 33.80 & 62.82 & 72.00 & 32.92 & 53.59 \\
&7/32& \jygl Ghosted Layers (Ours) & \jygl55.81 & \jygl37.03 & \jygl61.30 & \jygl67.25 & \jygl68.47 & \jygl37.00 & \jygl72.92 & \jygl81.00 & \jygl39.43 & \jygl57.80\\
\cmidrule{2-13}
&11/32& Shortened LLaMA  &  41.75 & 28.84 & 42.84 & 54.30 & 55.23 & 27.80 & 53.79 & 70.00 & 26.99 & 44.62
\\
&11/32& LLM-Streamline & 37.08 & 31.66 & 38.72 & 55.17 & 75.20 & 27.40 & 64.26 & 63.00 & 26.22 & 46.52 \\
&11/32& ShortGPT &  37.08 & 31.66 & 38.72 & 55.17 & 75.20 & 27.40 & 64.26 & 63.00 & 26.22 & 46.52 \\
&11/32& Prune\&Comp & 36.45 & 29.10 & 36.32 & 50.99 & 64.62 & 27.40 & 60.29 & 56.00 & 24.40 & 42.84 \\
&11/32& ReplaceMe (LS) &  40.19 & 30.63 & 44.70 & 61.40 & 66.45 & 32.00 & 74.73 & 65.00 & 33.78 & 49.88\\
&11/32& ReplaceMe (Cos) &  38.05 & 30.55 & 40.74 & 54.93 & 77.52 & 28.60 & 67.51 & 65.00 & 30.81 & 48.19\\
&11/32& LinearPatch (Diag) & 45.58 & 34.90 & 43.97 & 58.33 & 77.46 & 31.60 & 68.23 & 64.00 & 29.57 & 50.40 \\
&11/32& LinearPatch (Rotate) & 43.69 & 32.76 & 44.22 & 59.19 & 77.34 & 30.40 & 70.04 & 65.00 & 30.72 & 50.37 \\
&11/32& \jygl Ghosted Layers (Ours) & \jygl42.72 & \jygl32.59 & \jygl48.09 & \jygl64.33 & \jygl75.41 & \jygl31.80 & \jygl74.01 & \jygl66.00 & \jygl34.07 & \jygl52.11\\
\bottomrule
\end{tabular}
\end{adjustbox}
\label{tab:commonsense_extended}
\end{table*}

\setlength{\tabcolsep}{5.8pt}
\begin{table*}[t]
\centering
\footnotesize
\caption{Comparison on PPL benchmark with training-free methods over three LLMs.  $L_p/L_t$ denotes the number of pruned layers $L_p$ over the total number of layers $L_t$ in the original model.}
\begin{adjustbox}{width=1\columnwidth}
\begin{tabular}{l l r r r r r r r r r r r}
\toprule
&& \multicolumn{5}{c}{\textbf{7-layer}} && \multicolumn{5}{c}{\textbf{11-Layer}}\\
\cmidrule{3-7} \cmidrule{9-13}
\textbf{Model}  &\textbf{Method}  & \textbf{WIKI}\textcolor{red}{$\downarrow$} &  \textbf{C4}\textcolor{red}{$\downarrow$} & \textbf{PTB}\textcolor{red}{$\downarrow$} & \textbf{PPL AVG}\textcolor{red}{$\downarrow$} & \textbf{ACC AVG}\textcolor{red}{$\uparrow$}   &&   \textbf{WIKI}\textcolor{red}{$\downarrow$} & \textbf{C4}\textcolor{red}{$\downarrow$} & \textbf{PTB}\textcolor{red}{$\downarrow$} & \textbf{PPL AVG}\textcolor{red}{$\downarrow$} & \textbf{ACC AVG}\textcolor{red}{$\uparrow$}   \\
\midrule
\multirow{22}{*}{\begin{sideways}LLaMA-3-8B\end{sideways}}
& Dense & 5.47 & 6.71 & 22.51 & 11.56 & 67.51 && 5.47 & 6.71 & 22.51 & 11.56 & 67.51\\
\cmidrule{2-13}
& \textbf{Shortened LLaMA}  & 15.09 & 18.70 & 22.08 & 18.62 & 49.39 && 59.14 & 44.15 & 64.05 & 55.78 & 46.88 \\
& + Prune\&Comp & 12.67 & 17.62 & 20.30 & 16.86 & 49.74 && 29.55 & 31.87 & 41.91 & 34.44 & 45.50 \\
& + ReplaceMe (LS) & 68.72 & 58.23 & 121.37 & 82.77 & 43.52 && 203.92 & 180.54 & 289.82 & 224.76 & 37.40 \\
& + ReplaceMe (Cos) & 14.74 & 18.49 & 21.98 & 18.40 & 49.24 && 55.58 & 43.03 & 63.20 & 53.94 & 45.94 \\
& + LinearPatch (Diag)  & 12.27 & 17.42 & 20.35 & 16.68 & 50.13 && 23.91 & 28.38 & 37.14 & 29.81 & 47.33 \\
& + LinearPatch (Rotate)  & 12.35 & 17.20 & 20.27 & 16.61 & 49.39 && 25.34 & 29.02 & 37.78 & 30.71 & 46.73 \\
& \jygl + Ghosted Layers (Ours)  & \jygl 10.74 & \jygl 15.32 & \jygl 18.85 & \jygl 14.97 & \jygl 53.44 & \jygl& \jygl 18.95 & \jygl 23.65 & \jygl 37.45 & \jygl 26.68 & \jygl 47.84 \\
\cmidrule{2-13}
& \textbf{LLM-Streamline}    & 2305.48 & 2106.64 & 4642.40 & 3018.17 & 40.70 && 5594.89 & 4306.90 & 4481.35 & 4794.38 & 44.24 \\
& + Prune\&Comp & 172.95 & 248.38 & 237.30 & 219.54 & 46.55 && 449.61 & 375.50 & 745.48 & 523.53 & 40.34 \\
& + ReplaceMe (LS) & 30.02 & 26.24 & 52.83 & 36.36 & 59.45 && 133.34 & 78.58 & 380.05 & 197.32 & 51.66 \\
& + ReplaceMe (Cos) & 121.41 & 153.02 & 203.20 & 159.21 & 49.07 && 2158.67 & 1189.44 & 2896.85 & 2081.65 & 46.24 \\
& + LinearPatch (Diag) & 135.41 & 187.55 & 192.71 & 171.89 & 48.65 && 222.33 & 209.04 & 431.53 & 287.63 & 51.40 \\
& + LinearPatch (Rotate) & 79.08 & 110.39 & 92.50 & 93.99 & 50.55 && 272.37 & 224.54 & 470.82 & 322.58 & 51.37 \\
& \jygl + Ghosted Layers (Ours) & \jygl 21.31 & \jygl 21.62 & \jygl 40.64 & \jygl 27.86 & \jygl 60.10 & \jygl& \jygl 55.40 & \jygl 41.55 & \jygl 133.33 & \jygl 76.76 & \jygl 53.66 \\
\cmidrule{2-13}
\cmidrule{2-13}
& \textbf{ShortGPT} & 57.79 & 61.79 & 67.22 & 62.27 & 56.65 && 5594.89 & 4306.90 & 5481.35 & 5127.71 & 44.26 \\
& + Prune\&Comp & 31.20 & 41.84 & 49.95 & 41.00 & 54.66 && 644.80 & 495.67 & 1033.65 & 724.71 & 38.84 \\
& + ReplaceMe (LS) & 778.00 & 452.34 & 273.75 & 501.36 & 35.89 && 407.00 & 265.00 & 453.00 & 375.00 & 38.33 \\
& + ReplaceMe (Cos) & 79.65 & 66.60 & 151.45 & 99.23 & 56.93 && 784.38 & 889.44 & 520.97 & 731.60 & 39.52 \\
& + LinearPatch (Diag) & 28.40 & 38.21 & 38.20 & 34.94 & 58.12 && 207.20 & 195.15 & 410.81 & 271.05 & 51.57 \\
& + LinearPatch (Rotate)   & 35.30 & 40.46 & 43.33 & 39.70 & 58.77 && 439.50 & 403.85 & 950.21 & 597.85 & 49.71 \\
& \jygl + Ghosted Layers (Ours)  & \jygl 15.08 & \jygl 19.40 & \jygl 28.16 & \jygl 20.88 & \jygl 59.33 & \jygl& \jygl 57.13 & \jygl 44.14 & \jygl 161.57 & \jygl 87.61 & \jygl 53.52\\
\midrule
\multirow{22}{*}{\begin{sideways}LLaMA-3.1-8B\end{sideways}}
& Dense & 6.24 & 8.68 & 10.58 & 8.50 & 67.77 && 6.24 & 8.68 & 10.58 & 8.50 & 67.77 \\
\cmidrule{2-13}
& \textbf{Shortened LLaMA}  & 14.54 & 18.44 & 22.82 & 18.60 & 49.67 && 54.14 & 47.15 & 152.18 & 84.49 & 42.05 \\
& + Prune\&Comp & 29.41 & 40.46 & 44.26 & 38.04 & 42.60 && 182.71 & 119.57 & 253.54 & 185.27 & 41.22 \\
& + ReplaceMe (LS) & 70.03 & 55.41 & 134.37 & 86.60 & 43.84 && 1176.61 & 444.62 & 1827.31 & 1149.51 & 38.85 \\
& + ReplaceMe (Cos) & 14.40 & 18.35 & 22.68 & 18.48 & 49.79 && 46.80 & 38.71 & 58.73 & 48.08 & 42.05 \\
& + LinearPatch (Diag)  & 12.45 & 17.50 & 20.35 & 16.77 & 41.38 && 224.29 & 136.05 & 314.63 & 224.99 & 40.84 \\
& + LinearPatch (Rotate)  & 12.46 & 17.17 & 20.27 & 16.63 & 43.75 && 232.96 & 180.62 & 343.98 & 252.52 & 41.41 \\
& \jygl + Ghosted Layers (Ours)  & \jygl 10.79 & \jygl 15.26 & \jygl 18.26 & \jygl 14.77 & \jygl 53.55 & \jygl & \jygl 78.81 & \jygl 64.56 & \jygl 196.16 & \jygl 113.18 & \jygl 42.31 \\
\cmidrule{2-13}
& \textbf{LLM-Streamline}    & 2301.46 & 1173.53 & 3720.23 & 2398.41 & 42.60 && 4799.22 & 6510.32 & 6173.75 & 5827.76 & 44.68 \\
& + Prune\&Comp & 157.93 & 191.65 & 207.59 & 185.72 & 47.25 && 543.58 & 404.81 & 841.99 & 596.79 & 43.00 \\
& + ReplaceMe (LS) & 29.68 & 26.27 & 51.11 & 35.69 & 59.68 && 133.21 & 78.68 & 400.33 & 204.07 & 51.58 \\
& + ReplaceMe (Cos) & 147.97 & 131.94 & 195.69 & 158.53 & 49.75 && 1469.20 & 1487.28 & 1817.19 & 1591.22 & 46.63 \\
& + LinearPatch (Diag) & 110.93 & 134.68 & 135.50 & 127.04 & 50.46 && 224.32 & 195.13 & 391.71 & 270.39 & 52.39 \\
& + LinearPatch (Rotate) & 57.33 & 73.48 & 67.38 & 66.06 & 53.27 && 232.96 & 180.62 & 343.98 & 252.52 & 52.39 \\
& \jygl + Ghosted Layers (Ours) & \jygl 21.35 & \jygl 21.52 & \jygl 40.56 & \jygl 27.81 & \jygl 60.01 & \jygl & \jygl 54.98 & \jygl 41.31 & \jygl 125.73 & \jygl 74.01 & \jygl 53.92 \\
\cmidrule{2-13}
& \textbf{ShortGPT} & 63.42 & 69.64 & 68.58 & 67.21 & 57.10 && 3799.22 & 2510.32 & 4173.75 & 3494.43 & 44.71 \\
& + Prune\&Comp & 29.41 & 40.46 & 44.26 & 38.04 & 55.91 && 669.00 & 505.01 & 998.22 & 724.08 & 40.85 \\
& + ReplaceMe (LS) & 519.94 & 318.59 & 137.12 & 325.22 & 36.47 && 1253.00 & 592.00 & 286.00 & 710.33 & 37.48 \\
& + ReplaceMe (Cos) & 64.41 & 67.57 & 103.02 & 78.33 & 56.43 && 423.14 & 538.34 & 425.84 & 462.44 & 39.55 \\
& + LinearPatch (Diag) & 27.03 & 36.95 & 34.14 & 32.71 & 59.12 && 208.68 & 182.72 & 372.63 & 254.68 & 52.58 \\
& + LinearPatch (Rotate)   & 33.16 & 39.51 & 39.27 & 37.31 & 58.90 && 431.44 & 351.86 & 741.12 & 508.14 & 50.93 \\
& \jygl + Ghosted Layers (Ours)  & \jygl 15.72 & \jygl 19.73 & \jygl 27.01 & \jygl 20.82 & \jygl 60.45 & \jygl & \jygl 61.14 & \jygl 50.74 & \jygl 116.81 & \jygl 76.23 & \jygl 54.02 \\
\midrule
\multirow{22}{*}{\begin{sideways}DeepSeek-R1-Distill-LLaMA-8B\end{sideways}}
& Dense & 13.13 & 19.46 & 22.28 & 18.29 & 63.87 && 13.13 & 19.46 & 22.28 & 18.29 & 63.87\\
\cmidrule{2-13}
& \textbf{Shortened LLaMA}  & 29.26 & 37.19 & 45.26 & 37.24 & 49.00 && 94.06 & 87.80 & 127.20 & 103.02 & 44.62 \\
& + Prune\&Comp & 26.82 & 34.22 & 42.28 & 34.44 & 49.29 && 387.30 & 198.67 & 399.84 & 328.60 & 40.08 \\
& + ReplaceMe (LS) & 54.69 & 58.63 & 78.63 & 63.98 & 45.94 && 947.48 & 420.19 & 895.78 & 754.48 & 38.16 \\
& + ReplaceMe (Cos) & 27.60 & 35.12 & 42.89 & 35.20 & 49.75 && 95.91 & 84.17 & 134.11 & 104.73 & 45.21 \\
& + LinearPatch (Diag)  & 23.09 & 30.32 & 35.74 & 29.72 & 51.21 && 45.64 & 50.10 & 66.09 & 53.94 & 46.06\\
& + LinearPatch (Rotate)  & 23.66 & 30.84 & 36.32 & 30.27 & 51.69 && 45.08 & 48.64 & 63.73 & 52.48 & 45.90 \\
& \jygl + Ghosted Layers (Ours)  & \jygl 20.34 & \jygl 26.77 & \jygl 31.72 & \jygl 26.28 & \jygl 52.82 & \jygl & \jygl 33.66 & \jygl 39.54 & \jygl 46.21 & \jygl 39.80 & \jygl 47.25 \\
\cmidrule{2-13}
& \textbf{LLM-Streamline}    & 3083.95 & 1291.64 & 2985.68 & 2453.76 & 48.01 && 5094.57 & 4171.75 & 4114.49 & 4460.27 & 46.52 \\
& + Prune\&Comp & 805.43 & 521.87 & 1243.46 & 856.92 & 45.77 && 3061.56 & 1009.17 & 4305.52 & 2792.08 & 42.84 \\
& + ReplaceMe (LS) & 65.41 & 48.49 & 101.53 & 71.81 & 57.33 && 275.22 & 138.16 & 485.00 & 299.46 & 49.88 \\
& + ReplaceMe (Cos) & 441.74 & 312.40 & 578.42 & 444.19 & 51.69 && 4059.05 & 2408.55 & 8476.96 & 4981.52 & 48.19 \\
& + LinearPatch (Diag) & 476.26 & 341.89 & 596.51 & 471.55 & 51.68 && 1512.90 & 585.54 & 1998.35 & 1365.60 & 50.40 \\
& + LinearPatch (Rotate) & 319.22 & 212.18 & 459.40 & 330.27 & 53.59 && 2070.12 & 536.66 & 2752.07 & 1786.28 & 50.37 \\
& \jygl + Ghosted Layers (Ours) & \jygl 45.27 & \jygl 38.69 & \jygl 59.37 & \jygl 47.78 & \jygl 57.80 & \jygl & \jygl 115.18 & \jygl 75.07 & \jygl 181.76 & \jygl 124.00 & \jygl 52.11 \\
\cmidrule{2-13}
& \textbf{ShortGPT} & 343.44 & 157.21 & 906.19 & 468.95 & 55.05 && 4911.51 & 4261.16 & 3994.34 & 4389.00 & 46.54 \\
& + Prune\&Comp & 131.62 & 86.37 & 277.45 & 165.15 & 51.13 && 2901.15 & 1443.53 & 3941.85 & 2762.18 & 42.68  \\
& + ReplaceMe (LS) & 1225.19 & 793.62 & 3787.38 & 1935.40 & 42.78 && 4096.38 & 5405.75 & 4597.50 & 4699.88 & 38.89 \\
& + ReplaceMe (Cos) & 579.23 & 171.66 & 3628.31 & 1459.73 & 54.84 && 4505.67 & 4341.19 & 3178.56 & 4008.47 & 44.52 \\
& + LinearPatch (Diag) & 88.29 & 69.16 & 160.98 & 106.14 & 57.11 && 1502.95 & 577.91 & 1994.45 & 1358.44 & 50.37 \\
& + LinearPatch (Rotate) & 110.34 & 66.54 & 211.78 & 129.55 & 56.89 && 3654.27 & 847.01 & 4649.31 & 3050.20 & 49.45 \\
& \jygl + Ghosted Layers (Ours)  & \jygl 30.30 & \jygl 33.71 & \jygl 41.29 & \jygl 35.10 & \jygl 57.87 & \jygl& \jygl 131.41 & \jygl 81.61 & \jygl 201.41 & \jygl 138.14 & \jygl 52.31 \\
\bottomrule
\end{tabular}
\end{adjustbox}
\label{tab:ppl_extended}
\end{table*}

\begin{table*}[t]
\centering
\scriptsize
\caption{Zero-shot accuracy (\%) on nine commonsense QA benchmarks for three LLM backbones: OLMo-2-7B, LLaMA-2-7B, and Qwen3-14B. All methods are training-free. Ghosted Layers uses 32 C4 calibration sequences and LinearPatch and ReplaceMe use 128, with sequence length $T{=}2{,}048$. Recovery methods use LLM-Streamline. AVG denotes the mean accuracy across all nine tasks. \texttt{Ghosted Layers} achieves the highest average accuracy across all three backbones and both pruning ratios.}
\begin{adjustbox}{width=1\columnwidth}
\begin{tabular}{l  l  l c c c c c c c c c  r }
\toprule

\textbf{Model} & $L_p/L_t$ & \textbf{Method }&  \textbf{ARC-E}& \textbf{ARC-C }& \textbf{HellaS} & \textbf{WinoG} & \textbf{BoolQ}  & \textbf{OBQA} & \textbf{RTE} & \textbf{COPA} & \textbf{Race} & \textbf{AVG} $\uparrow$ \\
\midrule
\multirow{19}{*}{\begin{sideways}OLMo-2-7B\end{sideways}}
& 0/32 & Dense & 81.19 & 56.14 & 78.93 & 74.66 & 78.23 & 45.20 & 71.12 & 89.00 & 40.10 & 68.29\\
\cmidrule{2-13}
& 7/32 & Shortened LLaMA  & 71.34 & 43.26 & 68.16 & 65.75 & 65.20 & 41.20 & 62.09 & 79.00 & 33.30 & 58.81  \\
&7/32& LLM-Streamline &  66.41 & 40.78 & 64.20 & 70.32 & 69.63 & 36.00 & 71.48 & 82.00 & 37.70 & 59.84\\
&7/32& ShortGPT &  60.65 & 37.37 & 62.81 & 67.09 & 38.13 & 37.00 & 54.51 & 73.00 & 33.30 & 51.54\\
&7/32& Prune\&Comp &  43.56 & 26.37 & 44.68 & 57.30 & 54.68 & 28.80 & 49.10 & 69.00 & 27.08 & 44.51\\
&7/32& ReplaceMe (LS) &  65.70 & 39.59 & 63.14 & 70.96 & 70.21 & 36.40 & 70.04 & 81.00 & 38.09 & 59.46\\
&7/32& ReplaceMe (Cos) &  66.88 & 40.10 & 63.65 & 70.40 & 67.74 & 36.40 & 73.29 & 81.00 & 38.09 & 59.73\\
&7/32& LinearPatch (Diag) & 62.46 & 38.40 & 67.78 & 70.17 & 43.88 & 37.40 & 64.62 & 78.00 & 33.68 & 55.15\\
&7/32& LinearPatch (Rotate) & 67.34 & 41.64 & 67.01 & 71.03 & 63.94 & 37.20 & 67.87 & 80.00 & 36.27 & 59.14 \\
&7/32& \jygl Ghosted Layers (Ours) & \jygl 67.67 & \jygl 41.81 & \jygl 63.62 & \jygl 70.88 & \jygl 70.24 & \jygl 37.87 & \jygl 70.04 & \jygl 80.00 & \jygl 37.42 & \jygl 59.95\\

\cmidrule{2-13}
&11/32& Shortened LLaMA  & 53.45 & 30.29 & 52.32 & 54.70 & 62.23 & 32.20 & 53.79 & 66.00 & 31.58 & 48.51  \\
&11/32& LLM-Streamline &  50.17 & 33.96 & 52.64 & 67.01 & 56.33 & 33.60 & 78.34 & 71.00 & 30.43 & 52.61 \\
&11/32& ShortGPT &  36.78 & 25.43 & 33.59 & 51.85 & 38.32 & 26.80 & 50.54 & 66.00 & 25.93 & 39.47 \\
&11/32& Prune\&Comp &  28.79 & 25.68 & 27.36 & 50.75 & 39.27 & 27.00 & 47.65 & 60.00 & 20.86 & 36.37\\
&11/32& ReplaceMe (LS) &  46.34 & 31.91 & 46.22 & 65.90 & 68.04 & 31.40 & 73.65 & 70.00 & 30.33 & 51.53\\
&11/32& ReplaceMe (Cos) & 49.79 & 33.62 & 51.68 & 66.85 & 60.09 & 32.80 & 78.34 & 69.00 & 29.86 & 52.45 \\
&11/32& LinearPatch (Diag) & 39.86 & 30.55 & 42.43 & 56.99 & 51.19 & 34.80 & 71.84 & 64.00 & 26.99 & 46.52\\
&11/32& LinearPatch (Rotate) &  50.38 & 35.15 & 50.71 & 61.48 & 59.30 & 37.40 & 76.90 & 74.00 & 30.53 & 52.87\\
&11/32& \jygl Ghosted Layers (Ours) & \jygl 46.51 & \jygl 33.87 & \jygl 45.45 & \jygl 67.25 & \jygl 75.99 & \jygl 33.00 & \jygl 77.26 & \jygl 68.00 & \jygl 30.81 & \jygl 53.13 \\

\midrule

\multirow{19}{*}{\begin{sideways}LLaMA-2-7B\end{sideways}}
&0/32& Dense & 74.49 & 46.25 & 75.99 & 68.90 & 77.71 & 44.20 & 62.82 & 87.00 & 39.62 & 64.11 \\
\cmidrule{2-13}
&7/32& Shortened LLaMA  &  43.77 & 26.88 & 42.98 & 50.83 & 57.03 & 31.00 & 51.26 & 75.00 & 27.94 & 45.19\\
&7/32& LLM-Streamline &  48.61 & 32.76 & 56.15 & 64.48 & 62.17 & 32.80 & 57.40 & 77.00 & 32.25 & 51.51 \\
&7/32& ShortGPT &  48.61 & 32.76 & 56.15 & 64.48 & 62.17 & 32.80 & 57.40 & 77.00 & 32.25 & 51.51 \\
&7/32& Prune\&Comp &  48.44 & 31.91 & 53.81 & 59.83 & 62.17 & 35.80 & 57.04 & 83.00 & 31.96 & 51.55 \\
&7/32& ReplaceMe (LS) &  50.55 & 35.07 & 54.93 & 64.64 & 65.35 & 33.40 & 58.48 & 74.00 & 35.79 & 52.47 \\
&7/32& ReplaceMe (Cos) &  49.92 & 33.79 & 57.18 & 64.88 & 62.17 & 33.40 & 62.09 & 76.00 & 33.30 & 52.53 \\
&7/32& LinearPatch (Diag) & 55.13 & 34.56 & 57.12 & 63.46 & 62.17 & 35.60 & 55.96 & 78.00 & 35.02 & 53.00 \\
&7/32& LinearPatch (Rotate) &  55.22 & 33.70 & 57.92 & 65.19 & 62.14 & 35.60 & 55.60 & 78.00 & 34.64 & 53.11 \\
&7/32& \jygl Ghosted Layers (Ours) & \jygl 52.99 & \jygl 33.79 & \jygl 58.01 & \jygl 66.38 & \jygl 70.28 & \jygl 35.80 & \jygl 53.43 & \jygl 76.00 & \jygl 37.70 & \jygl 53.82\\
\cmidrule{2-13}
&11/32& Shortened LLaMA  &  44.95 & 25.09 & 42.44 & 51.07 & 47.77 & 30.40 & 54.87 & 72.00 & 27.08 & 43.96\\
&11/32& LLM-Streamline &  42.59 & 32.59 & 48.43 & 62.35 & 62.23 & 30.40 & 58.48 & 78.00 & 30.33 & 49.49\\
&11/32& ShortGPT & 42.80 & 30.46 & 44.50 & 60.46 & 62.26 & 35.40 & 47.29 & 72.00 & 29.57 & 47.19 \\
&11/32& Prune\&Comp & 42.68 & 28.16 & 42.79 & 57.70 & 62.26 & 30.00 & 52.71 & 78.00 & 25.93 & 46.69 \\
&11/32& ReplaceMe (LS) &  40.99 & 32.17 & 46.29 & 60.54 & 74.59 & 29.80 & 56.68 & 73.00 & 32.34 & 49.60 \\
&11/32& ReplaceMe (Cos) & 38.80 & 32.94 & 46.57 & 57.62 & 62.14 & 30.40 & 60.29 & 74.00 & 31.58 & 48.26 \\
&11/32& LinearPatch (Diag) & 48.95 & 33.53 & 53.35 & 63.06 & 62.20 & 34.00 & 59.93 & 77.00 & 34.07 & 51.79 \\
&11/32& LinearPatch (Rotate) & 49.28 & 32.51 & 52.73 & 62.98 & 62.17 & 33.20 & 61.73 & 77.00 & 33.88 & 51.72  \\
&11/32& \jygl Ghosted Layers (Ours) & \jygl 49.81 & \jygl 31.83 & \jygl 49.28 & \jygl 65.82 & \jygl 73.70 & \jygl 33.80 & \jygl 57.04 & \jygl 72.00 & \jygl 33.01 & \jygl 51.81\\

\midrule

\multirow{19}{*}{\begin{sideways}Qwen3-14B\end{sideways}}
&0/40& Dense &  82.83 & 60.24 & 78.82 & 72.85 & 89.30 & 46.20 & 77.62 & 90.00 & 43.16 & 71.22\\
\cmidrule{2-13}
&13/40& Shortened LLaMA  & 31.35 & 28.16 & 43.93 & 49.96 & 51.01 & 29.20 & 54.51 & 72.00 & 29.38 & 43.27  \\
&13/40& LLM-Streamline & 33.80 & 31.31 & 32.16 & 56.91 & 62.17 & 30.00 & 48.74 & 64.00 & 25.26 & 42.71\\
&13/40& ShortGPT & 29.59 & 26.71 & 37.82 & 49.49 & 62.02 & 28.80 & 49.82 & 61.00 & 24.88 & 41.13  \\
&13/40& Prune\&Comp & 33.80 & 31.31 & 32.16 & 56.91 & 62.17 & 30.00 & 48.74 & 64.00 & 25.26 & 42.71\\
&13/40& ReplaceMe (LS) & 32.32 & 27.47 & 30.69 & 58.17 & 78.17 & 27.80 & 56.32 & 64.00 & 26.41 & 44.59  \\
&13/40& ReplaceMe (Cos) & 33.16 & 30.72 & 32.48 & 57.30 & 62.17 & 30.00 & 49.10 & 64.00 & 25.26 & 42.69 \\
&13/40& LinearPatch (Diag) & 25.08 & 22.70 & 25.04 & 49.57 & 37.83 & 27.60 & 52.71 & 55.00 & 25.93 & 35.72\\
&13/40& LinearPatch (Rotate) & 25.08 & 22.70 & 25.04 & 49.57 & 37.83 & 27.60 & 52.71 & 55.00 & 25.93 & 35.72  \\
&13/40& \jygl Ghosted Layers (Ours) & \jygl 35.65 & \jygl 28.92 & \jygl 33.91 & \jygl 60.46 & \jygl 72.84 & \jygl 29.00 & \jygl 48.38 & \jygl 66.00 & \jygl 28.04 & \jygl 44.80 \\
\cmidrule{2-13}
&15/40& Shortened LLaMA  &  35.27 & 23.04 & 34.16 & 50.36 & 52.87 & 27.40 & 55.23 & 59.00 & 26.70 & 40.45\\
&15/40& LLM-Streamline & 39.48 & 27.39 & 35.85 & 48.22 & 38.50 & 29.00 & 50.18 & 61.00 & 21.05 & 38.96 \\
&15/40& ShortGPT & 29.04 & 26.54 & 33.49 & 48.70 & 61.68 & 28.60 & 52.35 & 55.00 & 23.83 & 39.91\\
&15/40& Prune\&Comp & 39.48 & 27.39 & 35.85 & 48.22 & 38.50 & 29.00 & 50.18 & 61.00 & 21.05 & 38.96\\
&15/40& ReplaceMe (LS) & 38.59 & 25.68 & 34.89 & 51.46 & 37.80 & 29.40 & 46.57 & 60.00 & 26.99 & 39.04 \\
&15/40& ReplaceMe (Cos) & 42.72 & 27.82 & 38.41 & 47.51 & 38.26 & 30.00 & 50.18 & 62.00 & 23.83 & 40.08\\
&15/40& LinearPatch (Diag) & 41.16 & 28.67 & 27.20 & 50.20 & 40.95 & 27.20 & 51.26 & 50.00 & 25.93 & 38.06 \\
&15/40& LinearPatch (Rotate) & 37.58 & 26.96 & 35.69 & 50.67 & 43.58 & 29.60 & 49.46 & 55.00 & 25.93 & 39.39\\
&15/40& \jygl Ghosted Layers (Ours) & \jygl 44.78 & \jygl 27.39 & \jygl 43.57 & \jygl 53.91 & \jygl 43.39 & \jygl 30.60 & \jygl 51.99 & \jygl 67.00 & \jygl 29.86 & \jygl 43.61 \\
\bottomrule
\end{tabular}
\end{adjustbox}
\label{tab:commonsense_additional_llms}
\end{table*}

%% file: contents/appendix/a-exp-wiki.tex
\label{sec:calib_dataset}

To evaluate the robustness of \texttt{Ghosted Layers} to the choice of calibration corpus, we repeat our zero-shot QA experiments on three 8B backbones with the calibration dataset replaced from C4~\citep{raffel2020c4} to WikiText-2~\citep{wiki}. All other settings, including the number of calibration sequences (128 for layer selection and baseline recovery, and 32 for Ghosted Layers operator estimation), sequence length ($T{=}2{,}048$), pruning criterion (LLM-Streamline~\citep{chen2024streamline}), and evaluation protocol across nine commonsense QA benchmarks, follow the main experiments reported in Table~\ref{tab:commonsense_main}.

Table~\ref{tab:commonsense_wiki} reports the results across three LLM backbones and two pruning ratios. \texttt{Ghosted Layers} attains the highest average accuracy in all six settings, consistent with the C4-calibrated results in the main paper. The average accuracy of \texttt{Ghosted Layers} changes by at most $1.09$ points when the calibration corpus is switched from C4 to WikiText-2 (e.g., $52.11 \to 53.20$ on DeepSeek-R1-Distill-LLaMA-8B at 11-layer pruning), indicating that the closed-form operator is not sensitive to the specific calibration distribution. Notably, the gap over the strongest training-free baseline widens under more aggressive 11-layer pruning, where the boundary activation mismatch is larger, suggesting that the unconstrained formulation is especially beneficial when the gap to reconstruct grows. This robustness to the calibration source is a practical advantage: practitioners can use whichever in-domain corpus is most readily available without retuning the recovery operator.

\begin{table*}[t]

\caption{Zero-shot accuracy (\%) on nine commonsense QA benchmarks for 7-layer and 11-layer pruning across three LLM backbones, using \textbf{WikiText-2} as the calibration corpus instead of C4. All other settings match Table~\ref{tab:commonsense_main}: 128 calibration sequences for layer selection and baseline recovery, and 32 sequences for \texttt{Ghosted Layers} operator estimation, with sequence length $T{=}2{,}048$ and pruning boundaries selected via LLM-Streamline. AVG denotes the mean accuracy across all nine tasks. \texttt{Ghosted Layers} attains the highest average accuracy in all six settings, demonstrating robustness to the choice of calibration corpus.}

\vspace{3pt}
\centering
\scriptsize
\begin{adjustbox}{width=1\columnwidth}
\begin{tabular}{l  l  l c c c c c c c c c  r }
\toprule

\textbf{Model} & \textbf{$L_p/L_t$} & \textbf{Method} &  \textbf{ARC-E} & \textbf{ARC-C} & \textbf{HellaS} & \textbf{WinoG} & \textbf{BoolQ}  & \textbf{OBQA} & \textbf{RTE} & \textbf{CoPa} & \textbf{Race} & \textbf{AVG} $\uparrow$ \\
\midrule
\multirow{19}{*}{\begin{sideways}LLaMA-3-8B\end{sideways}}
& 0/32 & Dense          &  77.78 & 53.24 & 79.16 & 72.53 & 81.38 & 45.00 & 69.68 & 89.00 & 40.19 & 67.55\\
\cmidrule{2-13}
&7/32& Shortened LLaMA  &  58.88 & 32.68 & 59.17 & 54.06 & 45.44 & 34.40 & 54.15 & 75.00 & 30.72 & 49.39\\
&7/32& LLM-Streamline &  39.65 & 29.18 & 33.23 & 55.41 & 38.04 & 29.80 & 57.40 & 60.00 & 24.02 & 40.75  \\
&7/32& ShortGPT &  39.65 & 29.18 & 33.23 & 55.41 & 38.04 & 29.80 & 57.40 & 60.00 & 24.02 & 40.75\\
&7/32& Prune\&Comp &  42.09 & 29.61 & 41.36 & 59.43 & 51.04 & 33.40 & 59.93 & 68.00 & 27.27 & 45.79\\
&7/32& ReplaceMe (LS) &  64.44 & 43.69 & 64.32 & 72.45 & 67.65 & 37.00 & 68.95 & 75.00 & 35.41 & 58.77\\
&7/32& ReplaceMe (Cos) &  50.93 & 34.73 & 49.71 & 66.14 & 39.02 & 34.00 & 63.18 & 66.00 & 29.67 & 48.15\\
&7/32& Linear Patch (D) & 43.18 & 31.66 & 43.14 & 60.62 & 57.31 & 33.80 & 64.98 & 68.00 & 28.52 & 47.91 \\
&7/32& Linear Patch (R) &  51.14 & 34.13 & 49.24 & 63.14 & 57.25 & 34.00 & 67.51 & 67.00 & 29.76 & 50.35\\
&7/32&  \jygl Ghost Layer (Ours) &  \jygl 65.70 &  \jygl 43.86 &  \jygl 66.33 &  \jygl 72.30 &  \jygl 69.60 &  \jygl 38.40 &  \jygl 67.87 &  \jygl 78.00 &  \jygl 36.94 &  \jygl 59.89\\
\cmidrule{2-13}
&11/32& Shortened LLaMA  & 45.50 & 29.69 & 48.81 & 52.64 & 59.42 & 29.20 & 49.82 & 67.00 & 26.60 & 45.41 \\
&11/32& LLM-Streamline & 38.17 & 29.86 & 32.93 & 56.75 & 56.09 & 30.20 & 70.04 & 57.00 & 27.27 & 44.26   \\
&11/32& ShortGPT &  38.17 & 29.86 & 32.93 & 56.75 & 56.09 & 30.20 & 70.04 & 57.00 & 27.27 & 44.26 \\
&11/32& Prune\&Comp & 37.25 & 28.16 & 42.86 & 57.14 & 62.81 & 29.20 & 56.32 & 62.00 & 25.07 & 44.53  \\
&11/32& ReplaceMe (LS) & 44.36 & 34.98 & 45.10 & 67.09 & 67.06 & 31.40 & 64.26 & 72.00 & 29.09 & 50.59 \\
&11/32& ReplaceMe (Cos) & 42.68 & 33.19 & 37.64 & 57.77 & 61.90 & 29.00 & 62.09 & 57.00 & 29.86 & 45.68 \\
&11/32& Linear Patch (D)  & 46.51 & 35.58 & 47.68 & 60.77 & 70.89 & 32.00 & 69.31 & 69.00 & 30.24 & 51.33 \\
&11/32& Linear Patch (R) & 47.81 & 34.64 & 44.03 & 60.85 & 76.09 & 31.60 & 66.43 & 67.00 & 31.96 & 51.16  \\
&11/32&  \jygl Ghost Layer (Ours) &  \jygl 49.12 &  \jygl 34.81 &  \jygl 51.16 &  \jygl 68.98 &  \jygl 77.31 &  \jygl 31.60 &  \jygl 66.43 &  \jygl 74.00 &  \jygl 32.34 &  \jygl 53.97\\
\midrule
\multirow{19}{*}{\begin{sideways}LLaMA-3.1-8B\end{sideways}}
&0/32 & Dense          &  81.31 & 53.50 & 78.90 & 73.72 & 82.02 & 44.80 & 69.68 & 87.00 & 39.23 & 67.80\\
\cmidrule{2-13}
&7/32& Shortened LLaMA  &   61.36 & 32.94 & 59.52 & 53.91 & 43.70 & 35.40 & 50.90 & 76.00 & 31.20 & 49.44\\
&7/32& LLM-Streamline & 44.15 & 33.02 & 33.40 & 56.83 & 38.20 & 32.60 & 58.12 & 61.00 & 26.03 & 42.59\\
&7/32& ShortGPT &  58.29 & 42.15 & 64.93 & 68.27 & 62.02 & 34.60 & 69.31 & 80.00 & 34.35 & 57.10\\
&7/32& Prune\&Comp & 45.20 & 30.46 & 44.05 & 58.41 & 51.41 & 34.80 & 60.65 & 67.00 & 27.46 & 46.60\\
&7/32& ReplaceMe (LS) & 66.67 & 43.09 & 64.23 & 73.01 & 65.72 & 35.00 & 71.12 & 77.00 & 35.41 & 59.03 \\
&7/32& ReplaceMe (Cos) & 55.81 & 36.09 & 49.86 & 63.69 & 39.33 & 36.00 & 60.29 & 69.00 & 30.91 & 49.00 \\
&7/32& Linear Patch (D) & 48.06 & 32.76 & 46.46 & 61.40 & 60.34 & 35.20 & 68.23 & 68.00 & 29.00 & 49.94 \\
&7/32& Linear Patch (R) & 58.00 & 37.54 & 54.21 & 64.17 & 60.49 & 35.40 & 69.68 & 69.00 & 29.09 & 53.06 \\
&7/32&  \jygl Ghost Layer (Ours) &  \jygl 68.69 &  \jygl 42.58 &  \jygl 66.21 &  \jygl 72.30 &  \jygl 66.09 &  \jygl 36.80 &  \jygl 71.48 &  \jygl 75.00 &  \jygl 37.70 &  \jygl 59.65\\
\cmidrule{2-13}
&11/32& Shortened LLaMA  &  48.53 & 29.61 & 49.52 & 52.72 & 59.36 & 29.60 & 51.26 & 64.00 & 26.51 & 45.68\\
&11/32& LLM-Streamline  & 39.65 & 30.29 & 31.47 & 56.51 & 55.08 & 30.20 & 69.68 & 61.00 & 28.52 & 44.71  \\
&11/32& ShortGPT & 39.65 & 30.29 & 31.47 & 56.51 & 55.08 & 30.20 & 69.68 & 61.00 & 28.52 & 44.71  \\
&11/32& Prune\&Comp &  42.47 & 29.27 & 42.88 & 56.20 & 63.79 & 29.00 & 59.93 & 61.00 & 27.37 & 45.77  \\
&11/32& ReplaceMe (LS) &  46.00 & 35.32 & 44.48 & 66.85 & 65.66 & 32.00 & 70.40 & 75.00 & 29.57 & 51.70 \\
&11/32& ReplaceMe (Cos) & 43.69 & 32.17 & 36.44 & 57.30 & 58.47 & 29.40 & 68.95 & 62.00 & 31.48 & 46.66 \\
&11/32& Linear Patch (D) & 50.51 & 36.77 & 48.29 & 62.51 & 70.61 & 31.60 & 72.56 & 68.00 & 29.76 & 52.29\\
&11/32& Linear Patch (R) &  50.80 & 35.15 & 46.60 & 61.88 & 75.29 & 31.00 & 70.04 & 70.00 & 30.72 & 52.39 \\
&11/32&  \jygl Ghost Layer (Ours) &  \jygl 50.46 &  \jygl 34.30 &  \jygl 50.62 &  \jygl 68.35 &  \jygl 74.65 &  \jygl 32.40 &  \jygl 67.87 &  \jygl 72.00 &  \jygl 33.11 &  \jygl 53.75 \\
\midrule
\multirow{19}{*}{\begin{sideways}DeepSeek-R1-Distill-LLaMA-8B\end{sideways}}
&0/32& Dense & 65.91 & 42.49 & 74.35 & 67.88 & 82.91 & 41.40 & 69.68 & 89.00 & 41.53 & 63.91\\
\cmidrule{2-13}
&7/32& Shortened LLaMA  & 48.48 & 30.97 & 50.84 & 50.75 & 62.72 & 30.60 & 58.48 & 69.00 & 31.67 & 48.17 \\
&7/32& LLM-Streamline &   49.12 & 37.12 & 55.98 & 63.22 & 77.06 & 34.00 & 74.73 & 71.00 & 33.21 & 55.05 \\
&7/32& ShortGPT &  49.12 & 37.12 & 55.98 & 63.22 & 77.06 & 34.00 & 74.73 & 71.00 & 33.21 & 55.05\\
&7/32& Prune\&Comp &  47.22 & 31.48 & 53.75 & 55.88 & 72.60 & 32.20 & 65.70 & 63.00 & 32.25 & 50.45\\
&7/32&  ReplaceMe (LS) &  54.34 &  37.37 &  60.04 &  66.30 &  73.61 &  33.40 &  72.56 &  81.00 &  40.10 &  57.64  \\
&7/32& ReplaceMe (Cos) & 51.47 & 36.09 & 58.12 & 62.67 & 76.27 & 31.00 & 75.45 & 71.00 & 34.16 & 55.14  \\
&7/32& Linear Patch (D) & 54.71 & 36.43 & 60.62 & 64.80 & 81.35 & 33.40 & 74.73 & 72.00 & 36.08 & 57.12\\
&7/32& Linear Patch (R) & 53.87 & 36.52 & 60.77 & 66.14 & 79.66 & 34.60 & 74.73 & 74.00 & 37.51 & 57.53 \\
&7/32& \jygl Ghost Layer (Ours) & \jygl 54.12 & \jygl 37.12 & \jygl 59.54 & \jygl 67.32 & \jygl 70.95 & \jygl 33.40 & \jygl 76.17 & \jygl 81.00 & \jygl 40.04 & \jygl 57.74 \\
\cmidrule{2-13}
&11/32& Shortened LLaMA  & 39.52 & 26.28 & 37.90 & 50.43 & 40.98 & 24.80 & 52.35 & 61.00 & 25.36 & 39.85 \\
&11/32& LLM-Streamline & 38.17 & 29.86 & 32.93 & 56.75 & 56.09 & 30.20 & 70.04 & 57.00 & 27.27 & 44.26 \\
&11/32& ShortGPT &  37.12 & 31.57 & 38.72 & 55.25 & 75.23 & 27.40 & 64.26 & 63.00 & 26.32 & 46.54\\
&11/32& Prune\&Comp &  35.98 & 29.35 & 36.54 & 51.93 & 64.53 & 25.60 & 59.57 & 56.00 & 24.69 & 42.69 \\
&11/32& ReplaceMe (LS) &  40.49 & 32.34 & 44.05 & 61.56 & 81.80 & 31.80 & 74.37 & 68.00 & 32.15 & 51.84\\
&11/32& ReplaceMe (Cos) &  38.47 & 31.06 & 40.85 & 54.22 & 77.77 & 27.20 & 67.15 & 65.00 & 30.14 & 47.98 \\
&11/32& Linear Patch (D) & 45.12 & 34.22 & 43.12 & 57.85 & 77.40 & 32.00 & 67.87 & 61.00 & 28.61 & 49.69\\
&11/32& Linear Patch (R) & 43.52 & 32.42 & 44.96 & 59.91 & 77.98 & 30.60 & 69.68 & 65.00 & 31.00 & 50.56 \\
&11/32&  \jygl Ghost Layer (Ours) &  \jygl 43.94 &  \jygl 33.36 &  \jygl 46.43 &  \jygl 62.83 &  \jygl 76.67 &  \jygl 33.20 &  \jygl 76.90 &  \jygl 71.00 &  \jygl 34.45 &  \jygl 53.20\\
\bottomrule
\end{tabular}
\end{adjustbox}
\label{tab:commonsense_wiki}
\end{table*}

%% file: contents/appendix/a-eff.tex
\label{app:efficiency}

This section details the methodology and full experimental results for the inference efficiency measurements reported in Table~\ref{tab:efficiency_llama31}.

\subsection{Setup}

\paragraph{Hardware.}
All measurements are conducted on a single NVIDIA A40 48GB GPU.

\paragraph{Model configuration.}
We evaluate LLaMA-3.1-8B under two pruning ratios, $n = 7$ and $n = 11$ out of $L = 32$ Transformer layers, in \texttt{float16}. Pruning boundaries are selected via LLM-Streamline~\citep{chen2024streamline} using 128 calibration sequences of length $T = 2{,}048$ sampled from the C4 training split. The selected boundaries are $[23, 30)$ for $n = 7$ and $[19, 30)$ for $n = 11$.

\paragraph{Latency protocol.}
For each configuration, we measure prefill latency with $3$ warmup iterations followed by $10$ timed runs. Each run performs a single forward pass with \texttt{use\_cache=False}, wrapped in \texttt{torch.cuda.synchronize()} before and after timing. We report the mean and standard deviation across the $10$ runs.

\paragraph{GPU protocol.}
Peak GPU memory is measured via \texttt{torch.cuda.max\_memory\_allocated()}, which records the maximum activated tensor footprint during the forward pass and is unaffected by PyTorch's caching allocator reserving unused memory blocks. Before each measurement, we reset the peak counter via \texttt{torch.cuda.reset\_peak\_memory\_stats()} and clear the allocator's cache through three cycles of \texttt{gc.collect()}, \texttt{torch.cuda.empty\_cache()}, and \texttt{torch.cuda.ipc\_collect()}, ensuring that the reported peak reflects only the memory required by the model under measurement.

\begin{table}[h]
\centering
\scriptsize
\caption{Prefill latency (ms) on LLaMA-3.1-8B with $n = 7$ pruned layers, across sequence lengths and batch sizes. Values are mean $\pm$ standard deviation over 10 runs with 3 warmup iterations; OOM indicates the configuration exceeded available GPU memory. \texttt{Ghosted Layers} consistently matches \texttt{LinearPatch} in latency across all configurations.}
\label{tab:latency_grid_n7}
\begin{tabular}{llrrrrr}
\toprule
Seq & Batch & Dense & Pruned (Streamline) & LinearPatch (Diag) & LinearPatch (Rotate) & Ours \\
\midrule
512  & 1  & $95.9 \pm 0.2$    & $76.6 \pm 0.4$    & $77.8 \pm 0.4$    & $78.8 \pm 0.8$    & $78.1 \pm 0.2$    \\
512  & 4  & $352.3 \pm 0.6$   & $281.8 \pm 1.6$   & $284.2 \pm 1.9$   & $285.7 \pm 1.6$   & $285.5 \pm 1.3$   \\
512  & 16 & $1329.0 \pm 1.0$  & $1055.4 \pm 3.1$  & $1064.6 \pm 0.9$  & $1067.8 \pm 2.4$  & $1071.0 \pm 3.0$  \\
\midrule
1024 & 1  & $185.7 \pm 1.0$   & $148.0 \pm 2.2$   & $148.1 \pm 1.0$   & $148.7 \pm 1.0$   & $149.4 \pm 0.9$   \\
1024 & 4  & $704.3 \pm 1.4$   & $558.9 \pm 1.1$   & $564.1 \pm 1.0$   & $564.8 \pm 1.1$   & $564.8 \pm 1.2$   \\
1024 & 16 & $2628.7 \pm 3.3$  & $2086.9 \pm 2.3$  & $2120.5 \pm 3.6$  & $2119.8 \pm 1.1$  & $2116.3 \pm 2.4$  \\
\midrule
2048 & 1  & $362.6 \pm 1.6$   & $287.8 \pm 1.4$   & $291.4 \pm 1.6$   & $291.7 \pm 1.7$   & $291.9 \pm 1.7$   \\
2048 & 4  & $1359.8 \pm 0.8$  & $1079.8 \pm 0.5$  & $1093.9 \pm 2.8$  & $1090.5 \pm 2.3$  & $1091.5 \pm 1.6$  \\
2048 & 16 & $5368.3 \pm 4.6$  & $4261.4 \pm 3.5$  & $4319.7 \pm 3.4$  & $4319.7 \pm 4.8$  & $4299.3 \pm 7.6$  \\
\midrule
4096 & 1  & $740.2 \pm 1.7$   & $587.8 \pm 1.0$   & $595.4 \pm 0.8$   & $594.5 \pm 1.3$   & $594.4 \pm 1.0$   \\
4096 & 4  & $2778.7 \pm 2.6$  & $2206.2 \pm 2.3$  & $2230.5 \pm 1.4$  & $2232.0 \pm 2.8$  & $2231.9 \pm 1.6$  \\
4096 & 16 & $11121.7 \pm 5.1$ & $8813.4 \pm 13.2$ & $8921.2 \pm 9.7$  & $8935.4 \pm 5.9$  & $8923.2 \pm 2.7$  \\
\midrule
8192 & 1  & $1506.0 \pm 0.5$  & $1193.2 \pm 1.2$  & $1205.3 \pm 1.3$  & $1206.2 \pm 0.9$  & $1205.5 \pm 0.5$  \\
8192 & 4  & $5952.6 \pm 4.2$  & $4718.1 \pm 2.0$  & $4773.3 \pm 5.3$  & $4771.7 \pm 3.8$  & $4767.2 \pm 3.6$  \\
\bottomrule
\end{tabular}
\end{table}

\begin{table}[h]
\centering
\scriptsize
\caption{Prefill latency (ms) on LLaMA-3.1-8B with $n = 11$ pruned layers, across sequence lengths and batch sizes. Same measurement protocol as Table~\ref{tab:latency_grid_n7}.}
\label{tab:latency_grid_n11}
\begin{tabular}{llrrrrr}
\toprule
Seq & Batch & Dense & Pruned (Streamline) & LinearPatch (Diag) & LinearPatch (Rotate) & Ours \\
\midrule
512  & 1  & $96.5 \pm 0.2$    & $65.5 \pm 0.2$    & $73.6 \pm 9.5$    & $67.0 \pm 0.2$    & $66.4 \pm 0.4$    \\
512  & 4  & $354.3 \pm 2.1$   & $238.6 \pm 0.7$   & $243.5 \pm 0.3$   & $243.3 \pm 1.5$   & $243.0 \pm 1.6$   \\
512  & 16 & $1334.9 \pm 2.0$  & $894.2 \pm 4.1$   & $913.0 \pm 2.9$   & $910.1 \pm 1.2$   & $913.4 \pm 3.1$   \\
\midrule
1024 & 1  & $186.6 \pm 1.7$   & $124.8 \pm 0.5$   & $127.1 \pm 1.1$   & $127.0 \pm 1.9$   & $126.8 \pm 0.9$   \\
1024 & 4  & $706.8 \pm 1.7$   & $476.2 \pm 0.9$   & $479.0 \pm 1.4$   & $479.8 \pm 0.3$   & $480.0 \pm 0.4$   \\
1024 & 16 & $2636.6 \pm 2.8$  & $1785.9 \pm 6.2$  & $1804.5 \pm 2.0$  & $1808.8 \pm 2.3$  & $1804.1 \pm 3.1$  \\
\midrule
2048 & 1  & $363.5 \pm 1.2$   & $244.4 \pm 1.4$   & $248.1 \pm 1.1$   & $248.1 \pm 1.0$   & $247.6 \pm 1.1$   \\
2048 & 4  & $1359.3 \pm 0.9$  & $917.4 \pm 2.0$   & $929.0 \pm 1.3$   & $929.9 \pm 1.0$   & $929.6 \pm 0.5$   \\
2048 & 16 & $5349.4 \pm 4.3$  & $3607.1 \pm 4.2$  & $3675.6 \pm 7.1$  & $3681.4 \pm 4.6$  & $3674.7 \pm 8.1$  \\
\midrule
4096 & 1  & $739.8 \pm 1.1$   & $497.9 \pm 1.0$   & $505.0 \pm 0.8$   & $505.6 \pm 1.0$   & $505.8 \pm 0.8$   \\
4096 & 4  & $2778.2 \pm 1.9$  & $1867.5 \pm 2.9$  & $1900.7 \pm 1.2$  & $1903.0 \pm 2.6$  & $1901.6 \pm 2.0$  \\
4096 & 16 & $11115.8 \pm 8.0$ & $7479.6 \pm 15.8$ & $7600.3 \pm 8.1$  & $7618.9 \pm 5.8$  & $7608.7 \pm 5.2$  \\
\midrule
8192 & 1  & $1507.1 \pm 0.4$  & $1015.3 \pm 0.7$  & $1026.4 \pm 0.8$  & $1027.7 \pm 1.2$  & $1027.1 \pm 0.7$  \\
8192 & 4  & $5945.9 \pm 2.1$  & $4017.7 \pm 3.1$  & $4059.2 \pm 1.8$  & $4066.5 \pm 4.5$  & $4063.6 \pm 4.9$  \\
\bottomrule
\end{tabular}
\end{table}

\subsection{Latency Analysis}

\paragraph{Aggregate metrics.}
Table~\ref{tab:efficiency_llama31} in the main paper reports the aggregate GPU memory, prefill latency, accuracy, and perplexity at sequence length $2{,}048$ and batch size $1$. \texttt{LinearPatch} (both Diag and Rotate variants) and \texttt{Ghosted Layers} exhibit nearly identical wall-clock latency, reflecting the structural equivalence established in Section~\ref{sec:analysis}: both reduce to a single $C \times C$ matrix multiplication at the boundary. At $n = 7$, \texttt{Ghosted Layers} ($291.9$\,ms) and \texttt{LinearPatch (Rotate)} ($291.7$\,ms) are indistinguishable within measurement noise, and they remain within $0.5$\,ms of each other at $n = 11$. 

\paragraph{Grid sweep across sequence lengths and batch sizes.}
To confirm that the cost equivalence holds beyond the representative configuration, we sweep prefill latency across sequence lengths $\{512, 1024, 2048, 4096, 8192\}$ and batch sizes $\{1, 4, 16\}$. Tables~\ref{tab:latency_grid_n7} and~\ref{tab:latency_grid_n11} report the $n = 7$ and $n = 11$ grids, respectively. Across both pruning ratios and all evaluated configurations, the latency of \texttt{Ghosted Layers} matches \texttt{LinearPatch} to within measurement noise and preserves the speedup gained from layer removal regardless of sequence length or batch size. This confirms that the inference-cost equivalence established analytically in Section~\ref{sec:analysis} holds robustly in practice.

%% file: contents/appendix/a-svd-computation.tex
\label{app:solver_comparison}
The closed-form operator $\mathbf{M}^{*} = \mathbf{X}_{\mathrm{pre}}^{\dagger}\boldsymbol{\Delta}$ in Theorem~\ref{thm:main_main} can be computed directly from the thin SVD of $\mathbf{X}_{\mathrm{pre}}$ via \colorbox{gray!15}{\texttt{torch.linalg.svd}}, or approximated by solving the regularized normal equations via \colorbox{gray!15}{\texttt{torch.linalg.solve}}. The latter yields a Tikhonov-regularized least-squares solution that converges to the unregularized pseudoinverse solution as $\epsilon \to 0$~\citep{golub2013matrix}. With the small $\epsilon = 10^{-6}$ used throughout, the two procedures produce nearly identical downstream results, as shown in Table~\ref{tab:solver_comparison}.

\paragraph{SVD-based.} Given the thin SVD $\mathbf{X}_{\mathrm{pre}} = \mathbf{U}\boldsymbol{\Sigma}\mathbf{V}^{\top}$ with $\mathbf{U} \in \mathbb{R}^{T_{\mathcal{D}} \times C}$, $\boldsymbol{\Sigma} \in \mathbb{R}^{C \times C}$, $\mathbf{V} \in \mathbb{R}^{C \times C}$ (the reduced form returned by {\texttt{torch.linalg.svd}} with \texttt{full\_matrices=False}), the Moore--Penrose pseudoinverse is $\mathbf{X}_{\mathrm{pre}}^{\dagger} = \mathbf{V}\boldsymbol{\Sigma}^{+}\mathbf{U}^{\top}$, yielding
\begin{equation}
    \mathbf{M}^{*} = \mathbf{V}\boldsymbol{\Sigma}^{+}\mathbf{U}^{\top}\boldsymbol{\Delta},
\end{equation}
where $\boldsymbol{\Sigma}^{+}$ contains the reciprocals of nonzero singular values. For numerical stability, our implementation truncates singular values below $10^{-6} \cdot \sigma_{\max}$~\citep{higham2002accuracy}. Internally, {\texttt{torch.linalg.svd}} computes the decomposition through a Jacobi-based driver (\texttt{gesvdj}) with a QR-based fallback (\texttt{gesvd}) on CUDA. This procedure is robust to rank deficiency, but explicitly materializes $\mathbf{U} \in \mathbb{R}^{T_{\mathcal{D}} \times C}$, whose size scales linearly with the calibration length $T_{\mathcal{D}}$.

\paragraph{Solver-based.} The minimum-norm solution can be approximated by solving the regularized normal equations
\begin{equation}
    (\mathbf{X}_{\mathrm{pre}}^{\top}\mathbf{X}_{\mathrm{pre}} + \epsilon \mathbf{I})\,\mathbf{M}_{\epsilon} = \mathbf{X}_{\mathrm{pre}}^{\top}\boldsymbol{\Delta},
    \qquad \epsilon = 10^{-6},
    \label{eq:normal_eq}
\end{equation}
via {\texttt{torch.linalg.solve}}, which solves the square linear system $\mathbf{A}\mathbf{X} = \mathbf{B}$ for invertible $\mathbf{A}$~\citep{golub2013matrix}. The small ridge term $\epsilon \mathbf{I}$ makes the system invertible and improves numerical conditioning, while $\mathbf{M}_{\epsilon}$ approaches $\mathbf{M}^{*}$ as $\epsilon \to 0$~\citep{higham2002accuracy}. Both accumulators $\mathbf{X}_{\mathrm{pre}}^{\top}\mathbf{X}_{\mathrm{pre}}$ and $\mathbf{X}_{\mathrm{pre}}^{\top}\boldsymbol{\Delta}$ live in $\mathbb{R}^{C \times C}$, with size independent of $T_{\mathcal{D}}$. They can be accumulated in a streaming fashion across calibration batches, so the entire calibration corpus need never reside in memory at once.

\begin{table*}[h]
\centering
\caption{SVD-based and solver-based computation of the recovery operator on LLaMA-3.1-8B with $n = 7$ pruned layers, using 32 calibration sequences ($T_{\mathcal{D}} = 65{,}536$, $C = 4{,}096$) in \texttt{float64}. Both procedures produce nearly identical downstream results. ACC AVG ($\uparrow$) denotes the mean zero-shot accuracy across the nine commonsense reasoning benchmarks used in our main experiments.}
\vspace{3pt}
\footnotesize
\begin{tabular}{l r r r r r}
\toprule
\textbf{Method }& \textbf{WIKI} &\textbf{C4} & \textbf{PTB} & \textbf{PPL AVG} & \textbf{Acc AVG}\\
\midrule
torch.linalg.svd & 21.35 & 21.53 & 40.56 & 27.81 & 60.00\\
torch.linalg.solve & 21.35 & 21.52 & 40.56 & 27.81 & 60.01 \\
\bottomrule
\end{tabular}
\label{tab:solver_comparison}
\end{table*}

As shown in Table~\ref{tab:solver_comparison}, the two procedures yield negligible differences in both perplexity and accuracy across all benchmarks, indicating that the small ridge regularization $\epsilon = 10^{-6}$ has negligible impact on downstream performance.

%% file: checklist.tex
\section*{NeurIPS Paper Checklist}

\begin{enumerate}

\item {\bf Claims}
    \item[] Question: Do the main claims made in the abstract and introduction accurately reflect the paper's contributions and scope?
    \item[] Answer: \answerYes{} 
    \item[] Justification: The abstract and introduction accurately describe the proposed Ghosted Layers method, the theoretical formulation of unconstrained activation alignment, and the empirical improvements observed across multiple LLM backbones and pruning settings.

\item {\bf Limitations}
    \item[] Question: Does the paper discuss the limitations of the work performed by the authors?
    \item[] Answer: \answerYes{} 
    \item[] Justification: The paper discusses limitations in Section 6, including the dependence on calibration data, access to the unpruned model, and the one-time offline computation required to estimate the recovery operator.

\item {\bf Theory assumptions and proofs}
    \item[] Question: For each theoretical result, does the paper provide the full set of assumptions and a complete (and correct) proof?
    \item[] Answer: \answerYes{} 
    \item[] Justification: The paper provides formal problem formulations and theoretical analysis for the unconstrained activation alignment objective. Proofs and derivations are included in the appendix and referenced in the main text.


    \item {\bf Experimental result reproducibility}
    \item[] Question: Does the paper fully disclose all the information needed to reproduce the main experimental results of the paper to the extent that it affects the main claims and/or conclusions of the paper (regardless of whether the code and data are provided or not)?
    \item[] Answer: \answerYes{} 
    \item[] Justification: The paper specifies the evaluated models, pruning methods, datasets, calibration setup, evaluation framework, and implementation details necessary to reproduce the reported experiments.


\item {\bf Open access to data and code}
    \item[] Question: Does the paper provide open access to the data and code, with sufficient instructions to faithfully reproduce the main experimental results, as described in supplemental material?
    \item[] Answer: \answerYes{} 
    \item[] Justification: The paper states that the official implementation will be released upon acceptance, and all experiments are conducted using publicly available models, datasets, and evaluation frameworks.

\item {\bf Experimental setting/details}
    \item[] Question: Does the paper specify all the training and test details (e.g., data splits, hyperparameters, how they were chosen, type of optimizer) necessary to understand the results?
    \item[] Answer: \answerYes{} 
    \item[] Justification: The paper provides detailed experimental settings, including benchmarks, datasets, pruning strategies, calibration sequences, hardware, and evaluation protocols in Sections 5 and the appendix.


\item {\bf Experiment statistical significance}
    \item[] Question: Does the paper report error bars suitably and correctly defined or other appropriate information about the statistical significance of the experiments?
    \item[] Answer: \answerNo{} 
    \item[] Justification: The paper reports deterministic evaluation metrics on fixed pretrained models and benchmarks, but does not include repeated runs or statistical significance analyses such as confidence intervals or error bars.


\item {\bf Experiments compute resources}
    \item[] Question: For each experiment, does the paper provide sufficient information on the computer resources (type of compute workers, memory, time of execution) needed to reproduce the experiments?
    \item[] Answer: \answerYes{} 
    \item[] Justification: The paper specifies the GPU type, latency measurement protocol, memory usage, and inference setup used for experiments.

    
\item {\bf Code of ethics}
    \item[] Question: Does the research conducted in the paper conform, in every respect, with the NeurIPS Code of Ethics \url{https://neurips.cc/public/EthicsGuidelines}?
    \item[] Answer: \answerYes{} 
    \item[] Justification: The authors have reviewed the NeurIPS Code of Ethics and believe the research complies with its guidelines.


\item {\bf Broader impacts}
    \item[] Question: Does the paper discuss both potential positive societal impacts and negative societal impacts of the work performed?
    \item[] Answer: \answerYes{} 
    \item[] Justification: The paper discusses both positive and negative societal impacts of improving the efficiency and deployability of large language models in the Broader Impact section.

    
\item {\bf Safeguards}
    \item[] Question: Does the paper describe safeguards that have been put in place for responsible release of data or models that have a high risk for misuse (e.g., pre-trained language models, image generators, or scraped datasets)?
    \item[] Answer: \answerNA{} 
    \item[] Justification: The paper does not release new pretrained language models or datasets, and the proposed method is a post-pruning recovery module applied to existing open-source models.


\item {\bf Licenses for existing assets}
    \item[] Question: Are the creators or original owners of assets (e.g., code, data, models), used in the paper, properly credited and are the license and terms of use explicitly mentioned and properly respected?
    \item[] Answer: \answerYes{} 
    \item[] Justification: The paper properly cites and credits all publicly used datasets, models, evaluation frameworks, and baseline methods throughout the paper and references.


\item {\bf New assets}
    \item[] Question: Are new assets introduced in the paper well documented and is the documentation provided alongside the assets?
    \item[] Answer: \answerYes{} 
    \item[] Justification: The paper introduces a new recovery method and states that the implementation will be publicly released upon acceptance with documentation and reproducibility details.


\item {\bf Crowdsourcing and research with human subjects}
    \item[] Question: For crowdsourcing experiments and research with human subjects, does the paper include the full text of instructions given to participants and screenshots, if applicable, as well as details about compensation (if any)? 
    \item[] Answer: \answerNA{} 
    \item[] Justification: The paper does not involve crowdsourcing or human subject research.


\item {\bf Institutional review board (IRB) approvals or equivalent for research with human subjects}
    \item[] Question: Does the paper describe potential risks incurred by study participants, whether such risks were disclosed to the subjects, and whether Institutional Review Board (IRB) approvals (or an equivalent approval/review based on the requirements of your country or institution) were obtained?
    \item[] Answer: \answerNA{} 
    \item[] Justification: The paper does not involve human subjects or data collection requiring IRB approval.


\item {\bf Declaration of LLM usage}
    \item[] Question: Does the paper describe the usage of LLMs if it is an important, original, or non-standard component of the core methods in this research? Note that if the LLM is used only for writing, editing, or formatting purposes and does \emph{not} impact the core methodology, scientific rigor, or originality of the research, declaration is not required.
    \item[] Answer: \answerNA{} 
    \item[] Justification: Large language models were used only for writing and editing assistance and were not involved in the development of the proposed methodology, experiments, theoretical analysis, or scientific conclusions.

\end{enumerate}